\documentclass{article}

\usepackage{microtype}
\usepackage{graphicx}
\usepackage{subcaption}
\usepackage{booktabs} 
\usepackage{multirow}
\usepackage{enumitem}

\usepackage{hyperref}

\usepackage[accepted]{icml2023}

\usepackage{amsmath}
\usepackage{amssymb}
\usepackage{mathtools}
\usepackage{amsthm}

\usepackage{algorithm}

\usepackage[noend]{algpseudocode}
\renewcommand{\algorithmicrequire}{\textbf{Input:}}
\renewcommand{\algorithmicensure}{\textbf{Output:}}
\newcommand{\Predict}{\textsc{Predict}}
\newcommand{\Train}{\textsc{Train}}

\usepackage[capitalize,noabbrev]{cleveref}

\theoremstyle{plain}
\newtheorem{theorem}{Theorem}[section]

\newtheorem{lemma}[theorem]{Lemma}

\theoremstyle{definition}
\newtheorem{definition}[theorem]{Definition}

\theoremstyle{remark}

\usepackage{dsfont}
\usepackage{xcolor}

\DeclareUnicodeCharacter{2212}{\textendash}

\usepackage[textsize=tiny]{todonotes}

\icmltitlerunning{Run-Off Election: Improved Provable Defense against Data Poisoning Attacks}

\newcommand{\fcertwo}{\texttt{Certv1}}
\newcommand{\fcerthree}{\texttt{Certv2}}
\newcommand{\CertFA}{\textsc{CertFA}}
\newcommand{\mycount}{\textnormal{count}}
\newcommand{\gapmean}{\gap^{+}}
\newcommand{\pw}{\textnormal{pw}}

\newcommand{\gap}{\textnormal{gap}}

\newcommand{\animalterm}[1]{\texttt{#1}}
\newcommand{\cat}{\animalterm{cat}}
\newcommand{\dog}{\animalterm{dog}}

\newcommand{\kpar}[1]{\textbf{#1.}}
\newcommand{\rone}{\textnormal{R1}}
\newcommand{\rtwo}{\textnormal{R2}}

\newcommand{\cpred}{c^{\textnormal{pred}}}
\newcommand{\csec}{c^{\textnormal{sec}}}
\newcommand{\hsplit}{h_{\text{spl}}}
\newcommand{\hspread}{h_{\text{spr}}}
\newcommand{\mD}{\mathcal{D}}
\newcommand{\mC}{\mathcal{C}}
\newcommand{\dsym}{d_{\text{sym}}}

\newcommand{\ind}[1]{\mathds{1}\left[ #1 \right]}
\newcommand{\logit}{\text{logits}}

\newcommand{\none}{N^{\text{\rone}}}
\newcommand{\nrun}{N^{\text{\rtwo}}}

\newcommand{\elect}{\textnormal{ROE}}
\newcommand{\cert}{\texttt{Cert}}

\newcommand{\ceil}[1]{\left\lceil #1 \right\rceil}

\DeclareMathOperator*{\argmax}{arg\,max}
\newcommand{\dparoe}{DPA+ROE}
\newcommand{\faroe}{FA+ROE}
\newcommand{\roe}{ROE}
\newcommand{\dpa}{DPA}
\newcommand{\dpastar}{DPA$^*$}
\newcommand{\dpastarroe}{DPA$^*$+ROE}
\newcommand{\fa}{FA}
\usepackage{todonotes}
\newcommand{\todok}[1]{}

\begin{document}

\twocolumn[
\icmltitle{Run-Off Election: Improved Provable Defense against Data Poisoning Attacks}



\icmlsetsymbol{equal}{*}

\begin{icmlauthorlist}
\icmlauthor{Keivan Rezaei}{equal,yyy}
\icmlauthor{Kiarash Banihashem}{equal,yyy}
\icmlauthor{Atoosa Chegini}{yyy}
\icmlauthor{Soheil Feizi}{yyy}
\end{icmlauthorlist}

\icmlaffiliation{yyy}{Department of Computer Science, University of Maryland, MD, USA}

\icmlcorrespondingauthor{Keivan Rezaei}{krezaei@umd.edu}

\icmlkeywords{Machine Learning, ICML}

\vskip 0.3in
]



\printAffiliationsAndNotice{\icmlEqualContribution} 

\begin{abstract}
    In data poisoning attacks,
    an adversary tries to change a model’s prediction by
    adding, modifying, or removing samples in the training data. 
    Recently,
    \emph{ensemble-based} approaches  for obtaining \emph{provable} defenses against data poisoning have been proposed where predictions are done by taking 
    a majority vote across multiple base models. 
  In this work, we show
  that merely considering the majority vote in ensemble defenses is wasteful as it does not effectively utilize available information in the logits layers of the base models.
  Instead, we propose \emph{Run-Off Election (ROE)},
  a novel aggregation method
  based on a two-round election across the base models: In the first round, models vote for their preferred class and then a second, \emph{Run-Off} election is held between the top two classes in the first round. Based on this approach, we propose \dparoe{} and \faroe{}
  defense methods based on Deep Partition Aggregation (DPA) and
  Finite Aggregation (FA) approaches from prior work.
  We evaluate our methods on MNIST, CIFAR-10, and GTSRB and obtain improvements in certified accuracy
  by up to $3\%$-$4\%$.
  Also, by applying ROE on a boosted version of DPA \footnote{We thank anonymous reviewer for proposing this method.},
  we gain improvements around 12\%-27\% comparing to the current state-of-the-art,
  establishing {\bf a new state-of-the-art} in (pointwise) certified robustness against data poisoning.
  In many cases, our approach outperforms the state-of-the-art, even when using 32 times less computational power.
\end{abstract}

\section{Introduction}
In recent years, Deep Neural Networks (DNNs) have achieved great success in many research areas, such as computer vision \cite{he2016deep} and natural language processing \cite{chen2015convolutional}
and have become the standard method of choice in many applications.
Despite this success, these methods are vulnerable to
\emph{poisoning attacks} where the adversary manipulates the training data
in order to change the classifications of specific inputs at the test time~\cite{chen2017targeted, shafahi2018poison, gao2021learning}.
Since large datasets are obtained 
using methods such as crawling the web,
this issue has become increasingly important as deep models
are adopted in safety-critical applications.

\looseness -1
While empirical defense methods have been proposed to combat this
problem using approaches such as data augmentation and data sanitization \cite{hodge2004survey, paudice2018label, gong2020maxup, borgnia2021strong, ni2021data},
the literature around poisoning has followed something of a ``cat and mouse" game
as in the broader literature on adversarial robustness,
where defense methods are quickly broken using adaptive and stronger attack techniques \cite{carlini2019evaluating}.
To combat this, several works have focused on obtaining \emph{certifiable defenses}
that are \emph{provably robust} against the adversary, regardless of the attack method.
These works provide a \emph{certificate} for each sample that is
a guaranteed lower bound on the amount
of distortion on the training set required to change
the model's prediction.

\begin{figure*}[ht!]
  \centering
  \includegraphics[width=1\linewidth]{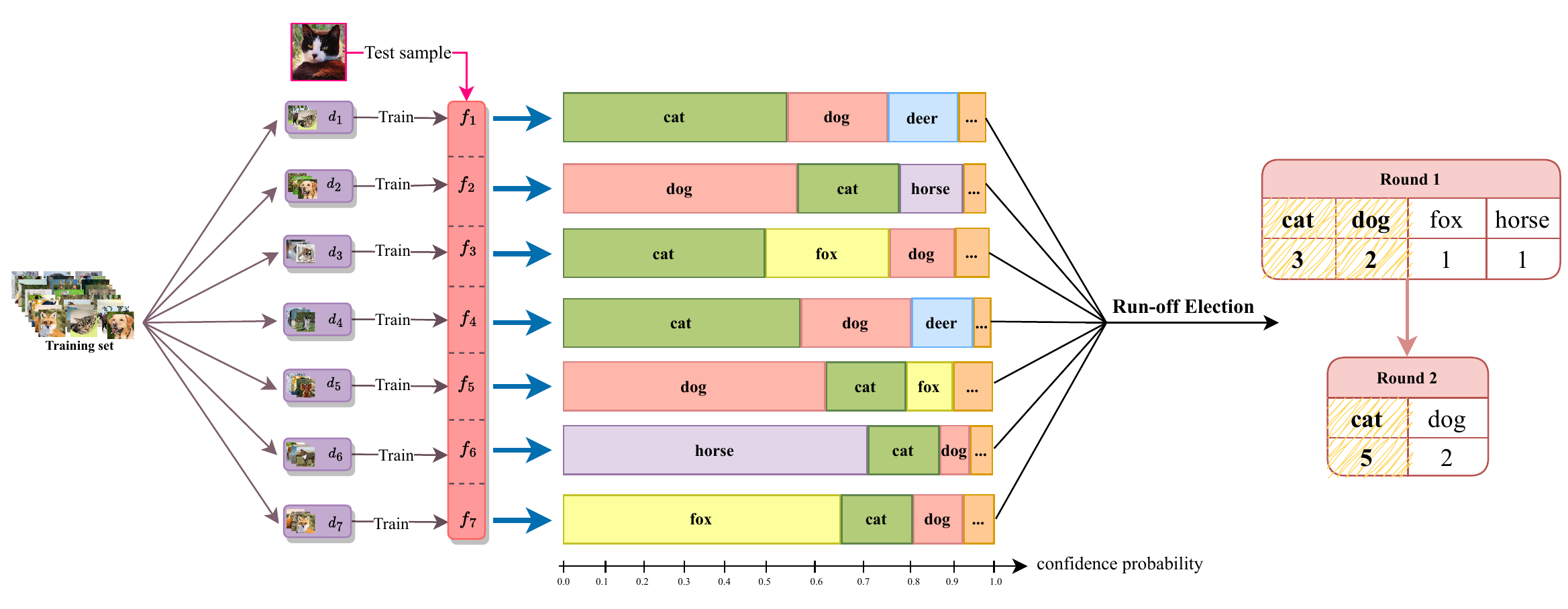}
  \caption{An example illustrating our proposed method \textbf{(Run-off Election)}.
    Training dataset is partitioned into 7 parts ($d_1, d_2, ..., d_7$) and $7$ separate classifiers are trained on each part ($f_1, f_2, ..., f_7$).
    At test time, after giving the input to the classifiers,
    we obtain the logits-layer information of each of them.
    For example, the class \dog{} has the highest logits-layer score for classifier $f_2$.
    In our method (see Section \ref{sec:method}), we hold a two-round election.
    In the first round, each model votes for its top class,
    and we find the top-two classes with the most votes.
    In the second run, each model votes for one of these two classes
    based on the logits-layer information, e.g., $f_6$ votes for \cat{}
    as it prefers it to \dog{}. 
    Existing methods output the most repeated class and can be fooled using a single poisoned sample,
    e.g., by changing the prediction of $f_3$ to \dog{}.
    As we prove in Section \ref{sec:roe_dpa_cert} (Theorem \ref{thm:certificate}), 
    our method is more robust and the adversary needs at least two poisoned samples to change the model's prediction in this example.
    This is due to the fact the \textbf{gap between the number of votes of the top two classes effectively increases in the second round.}
  }
  \label{fig:main_fig}
\end{figure*}

The most scalable provable defenses against data poisoning have been considered 
the use of
\emph{ensemble methods}
that
are composed of multiple base classifiers~\cite{levine2020deep,chen2020framework, jia2021intrinsic, wang2022improved, chen2022collective}.
At the test time, the prediction of these models
is aggregated by taking a majority vote across them.
Depending on the exact method, the certificates may be deterministic or stochastic.
For instance, Deep Partition Aggregation (DPA) method of
\cite{levine2020deep}
trains multiple models
on disjoint subsets of the training data.
Since each poisoned sample can affect at most one model, this leads
to a deterministic certificate based on the gap between
the predicted and the runner-up class. This can be wasteful, however as the models predicting other than the top two classes are ignored. 
While the choice of the partitioning scheme used
for training the models
has been extensively considered in the literature,  for both deterministic~\cite{levine2020deep, wang2022improved} and stochastic
~\cite{chen2020framework, jia2021intrinsic} partitioning schemes, all of these approaches share the problem as they take a  majority vote at test time.

In this work, we propose a novel aggregation method called Run-Off Election (ROE) that
greatly improves on existing approaches
using a \emph{two-round election} among the models.
In the first round, models vote for their preferred class, and we obtain the top two classes in terms of the number of votes.
We then hold a second, \emph{Run-Off} election round where
all base models vote for one of these two classes.
These votes are obtained using the \emph{logits layer} information of the models,
where each model votes for the class with the higher score in this layer.
By using all of the base models for prediction, we effectively 
increase the gap between the predicted class and the runner-up, leading to an improved certificate.
An~illustrative example of our approach is shown in Figure~\ref{fig:main_fig}. 
Our method is general and, in principle, can be applied
to any kind of deterministic or stochastic ensemble method.
In this paper, we focus on deterministic methods and develop \dparoe{} and \faroe{} 
defenses based on the
Deep Partition Aggregation (DPA)~\cite{levine2020deep} and Finite Aggregation (FA)~\cite{wang2022improved}
respectively and calculate the prediction certificates.

\looseness -1
\textbf{Technical challenges of calculating certificates.}
Compared to the majority vote method used in prior work,
calculating certificates in ROE  is more challenging  since characterizing the adversary's optimal actions is more complex.
Letting $\cpred$ and $\csec$ denote
the predicted class and
the runner-up class,
the adversary can change the model's prediction in many ways, as we briefly discuss below.

\vspace{-0.3cm}
\begin{enumerate}[leftmargin=*]
    \item It can get the $\csec$ elected in the second round by poisoning models
    to change their votes from the predicted class to the runner-up class.
    \item It can get $\cpred$ eliminated in the first round by ensuring that
    at least two classes receive more votes than $\cpred$ in the first round.
    \item For some $c\notin \{\cpred, \csec\}$, it can ensure that $c$ receives more votes than $\csec$ in the first round and receives more votes than $\cpred$ in the second round. This leads to the counter-intuitive situation where the adversary \emph{decreases} the votes of the runner-up class $\csec$.
\end{enumerate}
In order to obtain a unified argument for all of these cases, we introduce the concepts of a \emph{1v1 certificate} and \emph{2v1 certificate}. 
Intuitively, a 1v1 certificate bounds the number of poisoned samples required for one class to ``beat'' another class, i.e., receive more votes.
This is similar to the prediction certificate of
majority vote, with the distinction that we consider any two arbitrary classes.
A 2v1 certificate extends this idea and bounds the number of poisoned samples required for \emph{two} classes to beat another class simultaneously.

We will show that
as long as the 1v1 and 2v1 certificates can be calculated efficiently, we can use them to calculate a prediction certificate in ROE.
Taking a reduction approach is beneficial as it ensures that our method
works
for any choice of ensembles,
as long as 1v1 and 2v1 certificates can be calculated.
Focusing on \dparoe{} and \faroe{}, we show that the 1v1 certificates can be calculated using similar methods as the prediction certificates for the majority vote. Calculating 2v1 certificates are more complex, however as the adversary needs to ``spread'' its effort between two classes, and it is not straightforward how this should be done. 
For \dparoe{}, we use a \textbf{dynamic programming} based approach to recursively solve this problem. 
 The argument, however, does not extend to \faroe{}
since the underlying partitioning scheme is more complex
and the adversary is more constrained
in its actions.
For \faroe{}, we deal with this challenge using a \textbf{duality-based approach}
that considers the \textbf{convex combination of the adversary's constraint set}.

By reasoning on the adversary's behavior as above,
we obtain two separate certificates to ensure that the predicted class is unchanged
in each of the two rounds. 
Since the adversary can change our prediction
in either of the rounds,
we take the minimum of two numbers as our final certificate.
We further refer to Section~\ref{sec:roe_dpa_cert}
for more details on the calculation of the certificate
in \dparoe{} and \faroe{}.

\textbf{Empirical results.}
We evaluate our model in the context of
\emph{deterministic} robustness certificates
and observe substantial improvements over the existing state-of-the-art (FA).
\textbf{\faroe{} can improve certified accuracy by up to $4.79\%$, $4.73\%$, and $3.54\%$
respectively on MNIST, CIFAR-10, and GTSRB datasets.}
Furthermore, in some cases, \textbf{\dparoe{} also outperforms FA}
while it significantly uses less computational resources than FA. 
Note that FA improves over DPA by increasing the number of classifiers, which comes at the cost of a significant increase in training time.
Indeed, in some cases on all MNIST, CIFAR-10, and GTSRB datasets, \dparoe{} has improvements over FA while it exploits around \textbf{ 32 times less training cost}.
Finally, using a new ``boosted \dpa{}'' method called \dpastar{} that uses extra training cost to improve the base classifiers of \dparoe{}, we observe improvements of \textbf{up to $3.49\%$, $3.70\%$, and $3.44\%$
respectively on MNIST, CIFAR-10, and GTSRB datasets}, establishing a new state-of-the-art.
\footnote{
    We thank an anonymous referee for proposing this method.
    We note that the numbers reported show the difference in accuracy between \dpastarroe{} and \dpastar{} itself.
    Compared to FA, which was the previous state of the art, the improvements can be larger (as high as $27\%$).
    Depending on the setting, either \dpastarroe{}, or \faroe{} may produce more robust models.
    We refer to Section \ref{sec:experiment} for more details.
}

\kpar{Contributions}
In summary, our contributions include:
\begin{enumerate}
  \item
    We propose \emph{Run-Off election}, a novel aggregation method
    for ensemble-based defenses against data poisoning.
    Our approach is general, provable, and can be applied in combination
    with different partitioning schemes of the datasets.
    Using the partitioning schemes in DPA and FA,
    we propose the \dparoe{} and \faroe{} defense methods.
  \item
    We introduce the notion
    of 1v1 and 2v1 certificates and 
    show how they can be used to calculate \emph{provable certificates for robustness} for any ensemble method via a reduction.
    Focusing on \dparoe{} and \faroe{},
    we obtain these certificates using careful reasoning on the adversary's optimal action.
    For each round, we bound the minimum number of poisoned samples the adversary needs.
    In the first round, we propose a~dynamic~programming-based approach for characterizing
    the adversary's action in \dparoe{} and a duality-based approach
    for bounding the adversary's effect in \faroe{}. In the second round,
    we carefully bound the minimum number of samples required for electing other classes.
  \item We empirically evaluate our method on existing 
    benchmarks.
    Our experiments show that \roe{} consistently improves robustness for the \dpa{}, \fa{}, and \dpastar{} methods. 
    In some cases, we improve upon prior work
    even when using significantly fewer computational resources.
    Our method establishes the new \textbf{state-of-the-art} in certified accuracy against \textbf{general data poisoning attacks}.
\end{enumerate}

\subsection{Related work}
Certified robustness has been widely studied in the literature and
prior works have considered various notions of robustness, such as label-flipping \cite{rosenfeld2020certified}
and distributional robustness \cite{lai2016agnostic, DiakonikolasKKL16, diakonikolas2019sever}. 
Recent works have also studied the
poisoning problem theoretically using a PAC learning model
\cite{blum2021robust, gao2021learning, balcan2022robustly, hanneke2022optimal}.
In this work, we focus on (pointwise) certified robustness against data poisoning and assume a general poisoning model where the adversary may insert, delete, or modify any images.

Most closely related to our work, are the DPA \cite{levine2020deep} and the FA methods \cite{wang2022improved} that use an ensemble of classifiers to obtain \emph{deterministic} robustness certificates. A similar line of work \cite{chen2020framework, jia2021intrinsic} considers \emph{stochastic} robustness certificates. As mentioned, we improve on these works, establishing a new state-of-the-art for (pointwise) certified robustness. 
Following prior work, 
we use \emph{certified fraction},
i.e., 
the fraction of (test) data points
that are certifiable correct,
to measure robustness.
A similar but slightly different notion is studied by \cite{chen2022collective}
who
certify the \emph{test accuracy} without certifying any specific data point.

\looseness -1
Our work is closely related to the smoothing technique of \cite{cohen2019certified}.
that has been
extensively studied
both in terms of its applications~\cite{raghunathan2018semidefinite, singla2019robustness, singla2020second, chiang2020certified}
and known limitations~\cite{yang2020randomized, kumar2020curse, blum2020random}.
The original DPA method 
is inspired by derandomized smoothing~\cite{levine2020deep}.
Smoothing can also be directly applied to data poisoning attacks~\cite{weber2020rab}, though this requires strong assumptions on the adversary.

\section{Preliminaries}
\textbf{Notation.}
For a positive integer $n$,
we use $[n] := \{1, \dots, n\}$ to denote
the set of integers at most $n$.
Given two arbitrary sets $A$ and $B$,
we use $A \backslash B$ to denote
the set of all elements that are in $A$ but not in $B$
and use $A \times B$  to denote the \emph{Cartesian product},
i.e.,
$A \times B := \{(a, b) : a\in A, b\in B\}$.
We use $\dsym(A, B)$
to denote the size of
the \emph{symmetric difference} of $A$ and $B$,
i.e.,
\begin{align*}
  \dsym(A, B) := |(A \backslash B) \cup (B \backslash A)|
  .
\end{align*}
This can be thought of as a measure of distance between $A$ and $B$ and
equals the number of insertions and deletions required
for transforming $A$ to $B$.

We use $\mathcal{X}$ to denote the set of all possible unlabeled samples.
This is typically the set of all images, though our approach holds for any general input.
Similarly, we use $\mathcal{C}$ to be the set of possible labels.
We define a \emph{training set} $D$
as any arbitrary collection of labeled samples and let $\mathcal{D} := \mathcal{X} \times \mathcal{C}$.

We define a \emph{classification algorithm}
as any mapping $f : \mathcal{D} \times \mathcal{X} \to \mathcal{C}$,
where $f(D, x)$ denotes the prediction of the \emph{classifier} trained on the set $D$
and tested on the sample $x$.
We use the notation $f_{D}(x):= f(D, x)$ for convenience.
We assume that the classifier $f_D$ 
works by first
scoring each class and choosing the class with the maximum 
score.
For neural networks, this corresponds to the logits layer
of the model.
We use $f_D^{\logit}(x, c) \in \mathbb{R}$ to denote
the underlying score of class $c$ for the test sample $x$ and assume
that
$f_D^{\logit}(x, c) \ne f_D^{\logit}(x, c')$ for all $c\ne c'$.


\textbf{Threat model.}
We consider a \emph{general poisoning} model where
the adversary
can poison the training process by
adding, removing, or modifying the training set. 
Given an upper bound on the adversary's budget, i.e., the maximum amount of alteration it can make in the training set,
we aim to certify the prediction of the test samples.

Given a classification algorithm $f$, a dataset $D$ and a test sample $x$,
we define \emph{a prediction certificate}
as any
provable
lower bound on the number of samples
the adversary requires to change the prediction of $f$.
Formally,
\cert{} is a prediction certificate if
\begin{align*}
  f_{D}(x) = f_{D'}(x) \text{ if } \dsym(D, D') < \cert .
\end{align*}

\section{Proposed method: Run-Off election}
\label{sec:method}
In this section, we present our defense approach.
We start by discussing
\emph{Run-Off election}, an aggregation method
that takes as input a test sample $x$ and ensemble of $k$ models $\{f_i\}_{i=1}^{k}$,
and uses these models to make a prediction.
The method makes no assumptions about the ensemble and works for an arbitrary
choice of the models $f_i$.
In order to obtain certificates, however,
we will need to specify the choice of $f_i$.
In Section~\ref{sec:ensemble}, we consider two choices, \dparoe{} and \faroe{}.
In Section~\ref{sec:roe_dpa_cert}, we show how to obtain certificates for these methods.

\subsection{Run-Off Election}
\label{sec:roe}
As mentioned in the introduction,
our method can be seen as a \emph{two-round election}, where each model corresponds to a voter, and each class corresponds to a candidate.
Given a test sample $x$,
and an ensemble of $k$ models
$\{f_i\}_{i=1}^{k}$,
our election consists of the following two rounds.
\newcommand{\crone}{c_1^{R1}}
\newcommand{\cronep}{c_2^{R1}}
\begin{itemize}[leftmargin=*]
  \item \textbf{Round 1.}
    We first obtain the top two classes as measured
    by the number of models ``voting'' for each class.
    Formally, the setting,
    \begin{align}
        \none_c := \sum_i \ind{f_i(x) = c},
        \label{eq:def_n1}
    \end{align}
    we calculate
    the top two classes
    $\crone := \argmax_{c} \none_c$ and
    $\cronep := \argmax_{c\ne \crone} \none_c$.    
  \item \textbf{Round 2.}
    We collect the votes of each
    model in an~election between
    $\crone$ and $\cronep$. 
    Formally, for
    $(c, c') \in \{(\crone, \cronep), (\cronep, \crone)\}$, we set
    \begin{align*}
      \nrun_c := \sum_{i=1}^{k} \ind{f_i^{\logit}(x, c) > f_i^{\logit}(x, c')},
    \end{align*}
    and
    output $\elect(D, x) := \argmax_{c\in \{c_1, c_2\}} \nrun_c $.
\end{itemize}
We assume that $\argmax$ breaks ties by favoring the class with the smaller index. 
The formal pseudocode of ROE is provided in Algorithm~\ref{alg:ROE}.

\textbf{Reliability of the logits layer.}
Our method implicitly assumes that the information in the logits layer of the models is reliable enough to be useful for prediction purposes. This may seem counter-intuitive because many of the models do not even predict one of the top-two classes. In fact, these are precisely the models that were underutilized by \dpa{}. 
We note however that we only use the logits layer for a binary classification task in Round 2, which easier than multi-class classification.
Furthermore, even though the logits layer information may be less reliable than the choice of the top class, it is still useful because it provides \emph{more} information, making the attack problem more difficult for an adversary. As we will see in Section \ref{sec:experiment}, the clean accuracy of the model may slightly decrease when using our defense, but its robustness improves significantly for larger attack budgets.

\subsection{Choice of ensemble}
\label{sec:ensemble}
\looseness -1
We now present two possible choices of the ensembles $\{f_i\}_{i=1}$.
We begin by considering a disjoint partitioning scheme based on DPA and then consider a more sophisticated overlapping partitioning scheme based on FA. We denote the methods by \dparoe{} and \faroe{}, respectively.

\looseness -1
\textbf{\dparoe{}.}
In this method, training data is divided into several partitions
and a separate base classifier $f_i$ is trained on each of these partitions.
Formally, 
given a hash function $h: \mathcal{X} \to [k]$,
the training set $D$ is divided
into $k$ partitions $\{D_i\}_{i=1}^{k}$, where
$D_i := \{x\in D: h(x) = i\}$,
and the classifiers
$\{f_i\}_{i=1}^{k}$ are obtained
by training a base classifier on these partitions, i.e.,
$f_{i} := f_{D_i}$. For instance, when classifying images,
$f_i$ can be a standard ResNet model trained on $D_i$.

\textbf{\faroe{}.}
In this method,
we use two hash functions $\hsplit : \mD \to [kd]$ and $\hspread: [kd] \to [kd]^d$.
We first \emph{split} the datasets into $[kd]$ ``buckets'' using $\hsplit$. We then
create $kd$ partitions
by
\emph{spreading} these buckets, sending each bucket to all of the partitions specified
by $\hspread$. Formally,
for $i\in [kd]$, we define $D_i$ as
\begin{math}
  D_{i} := \{x \in D: i \in \hspread(\hsplit(x))\}.
\end{math}
We then train a separate classifier  $f_i := f_{D_i}$ for each
$D_i$.
A pseudocode of training FA is shown in Algorithm~\ref{alg:FA}.

\subsection{Calculating certificate}
\label{sec:roe_dpa_cert}
Since our aggregation method is more involved than taking a simple
majority vote, the adversary can affect the decision-making process in more ways.
Calculating the prediction certificate thus requires a more careful argument compared to prior work.
In order to present a unified argument,
we introduce the concept
of \emph{1v1} and \emph{2v1} certificates.
A 1v1 certificate bounds the number of poisoned samples required for one class to beat another class
while
a 2v1 certificate extends this idea and bounds the number of poisoned samples required for \emph{two} classes to beat another class.

We will show that
as long as the 1v1 and 2v1 certificates can be calculated efficiently, we can use them to calculate a~prediction certificate (Theorem~\ref{thm:certificate}).
The reduction ensures that
our approach is general and works
for any choice of ensemble,
as long as 1v1 and 2v1 certificates can be calculated.
We then provide an implementation of how to calculate those certificates for \dparoe{} (Lemmas~\ref{lem:certv1-dpa} and \ref{lem:certv2-dpa}) and \faroe{} (Lemmas~\ref{lem:certv1-fa} and \ref{lem:certv2-fa}).

We begin by defining the notion of the \emph{gap} between two classes.
\begin{definition}
    Given an ensemble
    $\{f_i\}_{i=1}^{k}$,
    a sample $x\in \mathcal{X}$,
    and classes $c, c'$, we
    define the \emph{gap} between $c, c'$ as
    \begin{align*}
        \gap(\{f_i\}_{i=1}^{k}, x, c, c') :=
         N_{c} - N_{c'} + \ind{c' > c},
    \end{align*}
    where $N_{c} := \sum_{i}\ind{f_i(x) = c}$
    For $\{f_i\}_{i=1}^{k}$ obtained
    using the training set $D$, 
    we use $\gap(D, x, c, c')$ to denote
    $\gap(\{f_i\}_{i=1}^{k}, x, c, c')$.
\end{definition}
We will omit the dependence of $\gap$ on $\{f_i\}_{i=1}^{k}$ and $x$ when it is clear from context.
We say that \emph{$c$ beats $c'$} if
$\gap(c, c') > 0$.
If the adversary wants
the class $c'$ to beat $c$, it needs
to poison the training set $D$ until $\gap(c, c')$ becomes non-positive.
We can therefore use this notion to reason on the optimal behaviour of the adversary.

We define a 1v1 certificate as follows.
\begin{definition}[1v1 certificate]
    \label{def:1v1}
    Given models $\{f_i\}_{i=1}^{k}$, a~test sample $x$, and two classes $c, c'\in \mC$
    we say $\fcertwo{} \in \mathbb{N}$ is a \emph{1v1 certificate} 
    for $c$ vs $c'$,
    if for all
    $D'$ such that
    $\dsym(D, D') < \fcertwo{}$,
    we have
    $\gap(D', x, c, c') > 0$.
\end{definition}
We note that if
$\gap(D, x, c, c')\le0$, then
$\fcertwo{}$ can only be zero.

We similarly define a 2v1 certificate as follows.
\begin{definition}[2v1 certificate]
    \label{def:2v1}
    Given models $\{f_i\}_{i=1}^{k}$, a~test sample $x$, and three classes $c, c_1, c_2\in \mC$
    we say $\fcerthree{} \in \mathbb{N}$ is a \emph{2v1 certificate} for $c$ vs $c_1, c_2$ if for all
    $D'$ such that
    $\dsym(D, D') < \fcerthree{}(\{f_i\}, x,c, c_1, c_2)$,
    we have
    $\gap(D', x, c, c_1) > 0$ 
    and
    $\gap(D', x, c, c_2) > 0$.
\end{definition}

Assuming these certificates can be calculated efficiently, we can obtain
a prediction certificate, as we outline below.
Let $\cpred$ denote the predicted and $\csec$ the runner-up class.
The adversary can change the model's prediction in one of the following two ways.

\begin{enumerate}[leftmargin=*]
    \item It can eliminate $\cpred$ in Round 1. This means it needs to choose two classes $c_1, c_2\in \mC \backslash \{\cpred\}$ and ensure
    that $c_1$ and $c_2$ both have more votes than $\cpred$ in Round 1. By definition, this requires as least $\fcerthree{}{}(\cpred, c_1, c_2)$. Since the adversary can choose $c_1, c_2$, we can lower bound the number
    of poisoned samples it needs with
    \begin{align*}
        \cert^{\rone} := 
        \min_{c_1, c_2\in \mC \backslash\{\cpred\}}
        \fcerthree{}{}(\cpred, c_1, c_2).
    \end{align*}
    \item It can eliminate $\cpred$ in Round 2. Letting $c$ denote the class that is ultimately chosen, this requires
    that $c$ makes it to Round 2 and beats $\cpred$ in Round 2.
    For $c$ to make it to Round 2, the adversary needs to ensure that it beats either $\cpred$ or $\csec$ in Round 1. 
    Given the previous case, we can assume that $\cpred$ makes it to Round 2, which means $c$ needs to beat $\csec$. 
    The number of poisoned samples required for this is at least
    \begin{align*}
        \cert^{\rtwo}_{c, 1} := \fcertwo{}(\{f_i\}_{i=1}^{k}, \csec, c).
    \end{align*}
    Note that this also includes the special case $c=\csec$ as $\fcertwo{}(\csec, \csec)=0$. 
    
    As for $c$ beating $\cpred$ in Round 2, let $g_i^{c} : \mathcal{X} \to \{c, \cpred\}$ denote the binary $\cpred$ vs $c$ classifier obtained from $f_i$.
    Formally,
    we set $g_i^c(x)$ to
    $c$ if 
    $f_i^{\logit}(x, c) > f_i^{\logit}(x, \cpred)$ and set
    it to $\cpred$ otherwise.
    We can lower bound the number of poisoned samples the adversary requires with
    \begin{align*}
        \cert^{\rtwo}_{c, 2} := 
        \fcertwo{}(\{g_i^{c}\}_{i=1}^{k}, \cpred, c)
        .
    \end{align*} 
    Overall, since the adversary can choose the class $c$, we obtain the bound
    \begin{align*}
        \cert^{\rtwo} := \min_{c \ne \cpred} \max\{\cert^{\rtwo}_{c, 1}, \cert^{\rtwo}_{c, 2}\}.
    \end{align*}
\end{enumerate}
Given the above analysis, we obtain the following theorem, a formal proof of which is provided in Appendix~\ref{pf_thm:cert}.
\begin{theorem}[\roe{} prediction certificate]\label{thm:certificate}
  \label{thm:cert}
  Let $\cpred$ denote the prediction of
  Algorithm \ref{alg:ROE}
  after training on a dataset $D$.
  For any training set $D'$, if
  \begin{math}
    \dsym(D, D') < \min
    \left\{
      \cert^{\rone},
      \cert^{\rtwo}
    \right\}
  ,
  \end{math}
  then Algorithm~\ref{alg:ROE} would still predict $\cpred$ when trained on the dataset $D'$.
\end{theorem}

We now show how we can obtain 1v1 and 2v1 certificates for \dparoe{} and \faroe{}.

\textbf{Certificate for \dparoe{}.}
We start with calculating $\fcertwo{}(c, c')$.
Since each poisoned sample can affect at most one model, the optimal action for the adversary is to ``flip'' a vote from $c$ to $c'$.
The adversary, therefore, requires at least
half of the gap between the votes of $c$ and $c'$. Formally, we use the following lemma, the proof of which is provided in Appendix~\ref{pf_lem:certv1-dpa}.
\begin{lemma}[\dparoe{} 1v1 certificate]
    \label{lem:certv1-dpa}
    Define \fcertwo{} as 
    \begin{math}
        \fcertwo{}(c, c') := 
        \ceil{\frac{\max\left(0, \gap(c, c')\right)}{2}}.
    \end{math}
    Then \fcertwo{} is a
    1v1 certificate for \dparoe{}.
\end{lemma}

\looseness -1
Calculating $\fcerthree{}(c, c_1, c_2)$ is more complex as the adversary's optimal action choice is less clear.
Intuitively, while the adversary should always change votes from $c$ to either $c_1$ or $c_2$, it is not clear
which class it should choose.
To solve this issue, we use dynamic programming. Defining $\gap_{i} := \max\{0, \gap(c, c_i)\}$ for $i\in \{1, 2\}$,
we calculate  $\fcerthree{}$ as a function of $\gap_{1}, \gap_{2}$.
As long as $\gap_{1}, \gap_{2} > 2$,
an optimal adversary should first choose a poison to reduce one of the gaps and then continue the poisoning process.
This leads to a recursive formulation which we solve efficiently using dynamic programming. 
\newcommand{\dpf}{\textnormal{dp}}
Formally, we fill a matrix
$\dpf$ of size $[k + 2]^2$
where if $\min(i, j) \ge 2$ we set
\begin{align*}
    \dpf[i, j] = 1 +
    \min\{\dpf[i-1, j-2], \dpf[i-2, j-1]\},
\end{align*}
and if $\min(i, j) \leq 1$, we set
$\dpf[i, j] := \ceil{\frac{\max(i, j)}{2}}$.
We obtain the following lemma, the proof of which is in Appendix~\ref{pf_lem:certv2-dpa}.
\begin{lemma}[\dparoe{} 2v1 certificate]
    \label{lem:certv2-dpa}  
    Define
    \begin{align*}
        \fcerthree{}(c, c_1, c_2)
        := 
        \dpf[\gap_1, \gap_2]
        ,
    \end{align*}
    where $\gap_i := \max\{0, \gap(c, c_i)\}$.
    Then \fcerthree{} is a
    2v1 certificate for \dparoe{}.
\end{lemma}

\textbf{Certificate for \faroe{}{}.}
We start with the 1v1 certificate. 
Consider a poisoned sample of the adversary and assume it falls in some bucket $i$.
By definition of buckets, this
can only affect the models $f_j$ satisfying
$j\in \hspread(i)$. 
If model $j$ votes for $c$, this can reduce the gap by at most two, and if the model votes for some class $\tilde{c} \notin \{c, c'\}$, it can reduce the gap by 1. 
This allows us to bound the effect of poisoning each bucket on the gap.
As we will see, the effect of poisoning multiple buckets is, at most,  the sum of the effects of each bucket.
Formally, we obtain the following lemma, the proof of which is in Appendix~\ref{pf_lem:certv1-fa}.
\begin{lemma}[\faroe{} 1v1 certificate]
    \label{lem:certv1-fa}
    Given two classes $c, c'$, 
    define the \emph{poisoning power} of each bucket $b\in [kd]$ as
    \begin{align*}
      \pw_{b} :=
      \sum_{i \in \hspread(b)} 2\ind{f_i(x) = c} + \ind{f_i(x) \notin \{c, c'\}}.
    \end{align*}
    Let $\fcertwo{}(c, c')$ be the smallest number such that the sum of the $\fcertwo{}$ largest values in $(\pw_{b})_{b\in [kd]}$ is at least $\gap(c, c')$. Then $\fcertwo{}{}$ is a 1v1 certificate.
\end{lemma}
Formal pseudocode for obtaining \fcertwo{} is provided in Algorithm~\ref{alg:fa-cert}.

In order to calculate $\fcerthree{}(c, c_1, c_2)$, we first observe that 
the adversary needs at least
$\max_{i}(\fcertwo{}(c, c_i))$ poisoned samples since both $c_1$ and $c_2$ need to beat $c$.
This is not necessarily enough, however, as making $c_1$ and $c_2$ beat $c$ \emph{simultaneously} may be more difficult than making each beat $c_i$ individually. In order to obtain a stronger bound, we use an approach inspired by duality
and consider the conical combination of the constraints. 
Defining $\gap_i := \max\{0, \gap(c, c_i)\}$, we observe that
if $\gap_1$ and $\gap_2$ both become non-positive, then so does every  combination $\lambda \gap_1 + \lambda' \gap_1$ for $\lambda, \lambda' \ge 0$. As a special case,
this includes
$\gapmean := \gap_1 + \gap_2$.
We can bound the number of poisoned samples for
making $\gapmean$ non-positive
using a similar argument as the 1v1 certificate.
Each bucket $b$ can only affect models $j$ such that $j\in \hspread(b)$. If $j$ votes
for $c_1$ or $c_2$, the gap cannot be reduced. 
If $j$ votes for $c$, the gap can be reduced by at most $3$ and
if $j$ votes for some $\tilde{c}\notin \{c, c_1, c_2\}$, the gap can be reduced by at most $1$. 
We Define the \emph{total poisoning power} of each bucket as
\begin{align*}
  \pw^{+}_{b} := 
  \sum_{i\in \hspread(b)} 3\ind{f_i(x) = c} + \ind{f_i(x) \notin (c, c_1, c_2)},
\end{align*}
where we hide the dependence on $c, c_1, c_2$ for brevity.
We obtain the following lemma, the proof of which is in Appendix \ref{pf_lem:certv2-fa}.
\begin{lemma}[\faroe{} 2v1 certificate]
    \label{lem:certv2-fa}
    For any $c, c_1, c_2 \in \mathcal{C}$,
    let $\fcerthree^{+}$ denote the smallest number such that the sum of the $\fcerthree^{+}$ largest values in $(\pw^{+}_{b})_{b\in [kd]}$ is at least $\gapmean$.
    For $i\in \{1, 2\}$,
    define $\fcerthree^{(i)} := \fcertwo{}(c, c_i)$. 
    Finally, define $\fcerthree{}$ as
    \begin{align*}
        \fcerthree := \max\{\fcerthree^{(1)}, \fcerthree^{(2)}, \fcerthree^{+}\}.
    \end{align*}
    Then \fcerthree{} is a 2v1 certificate for \faroe{}.
\end{lemma}

\section{Evaluation}
\label{sec:experiment}

\begin{table*}[tb]
    \caption{
        Certified fraction of \dpastarroe{}, \faroe{}, and original FA with various values of hyperparameter $d$
        with respect to different attack sizes $B$.
        Improvements of \dpastarroe{} and \faroe{} over the original FA (with same $d$) are highlighted in blue if they are positive and red otherwise.
    }
    \vskip 0.15in
    \label{table:roe_fa_vs_fa}
    \resizebox{\linewidth}{!}{
    \begin{tabular}{|c||c||c|c||c|c|c|c|c|}
    \hline
        dataset & $k$ & method & $d$ & \multicolumn{5}{|c|}{certified fraction}\\ 
        \hline\hline
        \multirow{5}{*}{MNIST}      & \multirow{4}{*}{1200}    & \multicolumn{2}{c||}{}                    & $B\leq 100$ & $B\leq 200$ & $B\leq 300$ & $B\leq 400$ & $B\leq 500$\\\cline{3-9}
                                    &                          & \footnotesize{FA}             & \multirow{3}{*}{16}        & 92.75\% & 87.89\% & 78.91\% & 62.42\% & 31.97\%\\\cline{5-9}
                                    &                          & \footnotesize{\faroe{}}       &         & 92.80\%{\color{blue}(+0.05\%)} & 88.09\%{\color{blue}(+0.2\%)} & 80.26\%{\color{blue}(+1.35\%)} & 65.31\%{\color{blue}(+2.89\%)} & 36.76\%{\color{blue}(+4.79\%)} \\\cline{5-9}
                                    &                          & \footnotesize{\dpastarroe{}}       &         & 92.70\%{\color{red}(-0.05\%)} & 88.41\%{\color{blue}(+0.52\%)} & 81.75\%{\color{blue}(+2.84\%)} & 69.67\%{\color{blue}(+7.25\%)} & 44.81\%{\color{blue}(+12.84\%)} \\\cline{3-9}
        \hline\hline
        \multirow{10}{*}{CIFAR-10}  & \multirow{7}{*}{50}      & \multicolumn{2}{c||}{}                    & $B\leq 5$ & $B\leq 10$ & $B\leq 15$ & $B\leq 18$ & $B\leq 20$\\\cline{3-9}
                                    &                          & \footnotesize{FA}             & \multirow{3}{*}{16}        & 60.55\% & 48.85\% & 34.61\% & 25.46\% & 19.90\%\\\cline{5-9}
                                    &                          & \footnotesize{\faroe{}}       &         & 61.71\%{\color{blue}(+1.16\%)} & 51.18\%{\color{blue}(+2.33\%)} & 37.12\%{\color{blue}(+2.51\%)} & 28.49\%{\color{blue}(+3.03\%)} & 22.08\%{\color{blue}(+2.18\%)}\\\cline{5-9}
                                    &                          & \footnotesize{\dpastarroe{}}  &         & 61.87\%{\color{blue}(+1.32\%)} & 52.71\%{\color{blue}(+3.86\%)} & 41.51\%{\color{blue}(+6.9\%)} & 33.42\%{\color{blue}(+7.96\%)} & 27.47\%{\color{blue}(+7.57\%)} \\\cline{3-9}
                                    &                          & \footnotesize{FA}             & \multirow{3}{*}{32}        & 61.31\% & 50.31\% & 36.03\% & 26.55\% & 19.93\%\\\cline{5-9}
                                    &                          & \footnotesize{\faroe{}}       &         & 62.56\%{\color{blue}(+1.25\%)} & 52.55\%{\color{blue}(+2.24\%)} & 38.83\%{\color{blue}(+2.8\%)} & 29.05\%{\color{blue}(+2.5\%)} & 21.97\%{\color{blue}(+2.04\%)}\\\cline{5-9}
                                    &                          & \footnotesize{\dpastarroe{}}  &         & 65.99\%{\color{blue}(+4.68\%)} & 61.51\%{\color{blue}(+11.2\%)} & 56.13\%{\color{blue}(+20.1\%)} & 51.83\%{\color{blue}(+25.28\%)} & 47.61\%{\color{blue}(+27.68\%)}\\\cline{2-9}
                                    & \multirow{7}{*}{250}     & \multicolumn{2}{c||}{}                    & $B\leq 10$ & $B\leq 20$ & $B\leq 40$ & $B\leq 50$ & $B\leq 60$\\\cline{3-9}
                                    &                          & \footnotesize{FA}             & \multirow{3}{*}{8}         & 45.38\% & 36.05\% & 20.08\% & 14.39\% & 9.70\%\\\cline{5-9}
                                    &                          & \footnotesize{\faroe{}}       &          & 46.80\%{\color{blue}(+1.42\%)} & 38.56\%{\color{blue}(+2.51\%)} & 23.61\%{\color{blue}(+3.53\%)} & 17.86\%{\color{blue}(+3.47\%)} & 13.06\%{\color{blue}(+3.36\%)}\\\cline{5-9}
                                    &                          & \footnotesize{\dpastarroe{}{}}       &          & 47.14\%{\color{blue}(+1.76\%)} & 39.32\%{\color{blue}(+3.27\%)} & 25.41\%{\color{blue}(+5.33\%)} & 19.68\%{\color{blue}(+5.29\%)} & 15.02\%{\color{blue}(+5.32\%)}\\\cline{3-9}
                                    &                          & \footnotesize{FA}             & \multirow{3}{*}{16}        & 46.52\% & 37.56\% & 21.99\% & 15.79\% & 11.09\%\\\cline{5-9}
                                    &                          & \footnotesize{\faroe{}}       &         & 48.33\%{\color{blue}(+1.81\%)} & 40.71\%{\color{blue}(+3.15\%)} & 26.38\%{\color{blue}(+4.39\%)} & 20.52\%{\color{blue}(+4.73\%)} & 14.64\%{\color{blue}(+3.55\%)}\\\cline{5-9}
                                    &                          & \footnotesize{\dpastarroe{}}       &         & 46.88\%{\color{blue}(+0.36\%)} & 39.50\%{\color{blue}(+1.94\%)} & 25.49\%{\color{blue}(+3.5\%)} & 19.83\%{\color{blue}(+4.04\%)} & 15.04\%{\color{blue}(+3.95\%)}\\\cline{3-9}
        \hline\hline
        \multirow{10}{*}{GTSRB}     & \multirow{7}{*}{50}      & \multicolumn{2}{c||}{}                    & $B\leq 5$ & $B\leq 10$ & $B\leq 15$ & $B\leq 20$ & $B\leq 22$\\\cline{3-9}
                                    &                          & \footnotesize{FA}             & \multirow{3}{*}{16}        & 82.71\% & 74.66\% & 63.77\% & 47.52\% & 35.54\%\\\cline{5-9}
                                    &                          & \footnotesize{\faroe{}}       &         & 82.59\%{\color{red}(--0.12\%)} & 75.55\%{\color{blue}(+0.89\%)} & 65.47\%{\color{blue}(+1.7\%)} & 50.33\%{\color{blue}(+2.81\%)} & 38.89\%{\color{blue}(+3.35\%)}\\\cline{5-9}
                                    &                          & \footnotesize{\dpastarroe{}}       &         & 83.67\%{\color{blue}(+0.96\%)} & 77.84\%{\color{blue}(+3.18\%)} & 70.63\%{\color{blue}(+6.86\%)} & 57.97\%{\color{blue}(+10.45\%)} & 49.33\%{\color{blue}(+13.79\%)}\\\cline{3-9}
                                    
                                    &                          & \footnotesize{FA}             & \multirow{3}{*}{32}        & 83.52\% & 76.26\% & 66.32\% & 49.68\% & 38.31\%\\\cline{5-9}
                                    &                          & \footnotesize{\faroe{}}       &         & 83.61\%{\color{blue}(+0.1\%)} & 77.07\%{\color{blue}(+0.81\%)} & 67.83\%{\color{blue}(+1.51\%)} & 51.81\%{\color{blue}(+2.14\%)} & 41.61\%{\color{blue}(+3.29\%)}\\\cline{5-9}
                                    &                          & \footnotesize{\dpastarroe{}}       &         & 83.67\%{\color{blue}(+0.15\%)} & 77.93\%{\color{blue}(+1.66\%)} & 70.67\%{\color{blue}(+4.35\%)} & 58.38\%{\color{blue}(+8.71\%)} & 49.61\%{\color{blue}(+11.3\%)}\\\cline{2-9}
                                    & \multirow{7}{*}{100}     & \multicolumn{2}{c||}{}                    & $B\leq 5$ & $B\leq 15$ & $B\leq 20$ & $B\leq 25$ & $B\leq 30$\\\cline{3-9}
                                    &                          & \footnotesize{FA}             & \multirow{3}{*}{16}        & 48.19\% & 33.95\% & 25.96\% & 18.92\% & 13.82\%\\\cline{5-9}
                                    &                          & \footnotesize{\faroe{}}       &         & 48.00\%{\color{red}(--0.19\%)} & 35.76\%{\color{blue}(+1.81\%)} & 28.92\%{\color{blue}(+2.95\%)} & 22.30\%{\color{blue}(+3.38\%)} & 16.32\%{\color{blue}(+2.49\%)}\\\cline{5-9}
                                    &                          & \footnotesize{\dpastarroe{}}       &         & 47.50\%{\color{red}(--0.69\%)} & 36.48\%{\color{blue}(+2.53\%)} & 30.60\%{\color{blue}(+4.64\%)} & 24.63\%{\color{blue}(+5.71\%)} & 19.26\%{\color{blue}(+5.44\%)}\\\cline{3-9}
                                    &                          & \footnotesize{FA}             & \multirow{3}{*}{32}        & 48.39\% & 34.96\% & 27.05\% & 19.83\% & 14.47\%\\\cline{5-9}
                                    &                          & \footnotesize{\faroe{}}       &         & 48.15\%{\color{red}(--0.25\%)} & 36.81\%{\color{blue}(+1.84\%)} & 29.85\%{\color{blue}(+2.8\%)} & 23.37\%{\color{blue}(+3.54\%)} & 17.41\%{\color{blue}(+2.95\%)}\\\cline{5-9}
                                    &                          & \footnotesize{\dpastarroe{}}       &         & 47.66\%{\color{red}(--0.74\%)} & 36.66\%{\color{blue}(+1.69\%)} & 30.70\%{\color{blue}(+3.66\%)} & 24.71\%{\color{blue}(+4.88\%)} & 19.33\%{\color{blue}(+4.87\%)}\\\cline{3-9}
        \hline
    \end{tabular}}
\end{table*}

\begin{figure*}[tbp]
  \centering
  \begin{subfigure}{1\linewidth}
    \centering
    \begin{minipage}{0.31\linewidth}
      \includegraphics[width=1\linewidth]{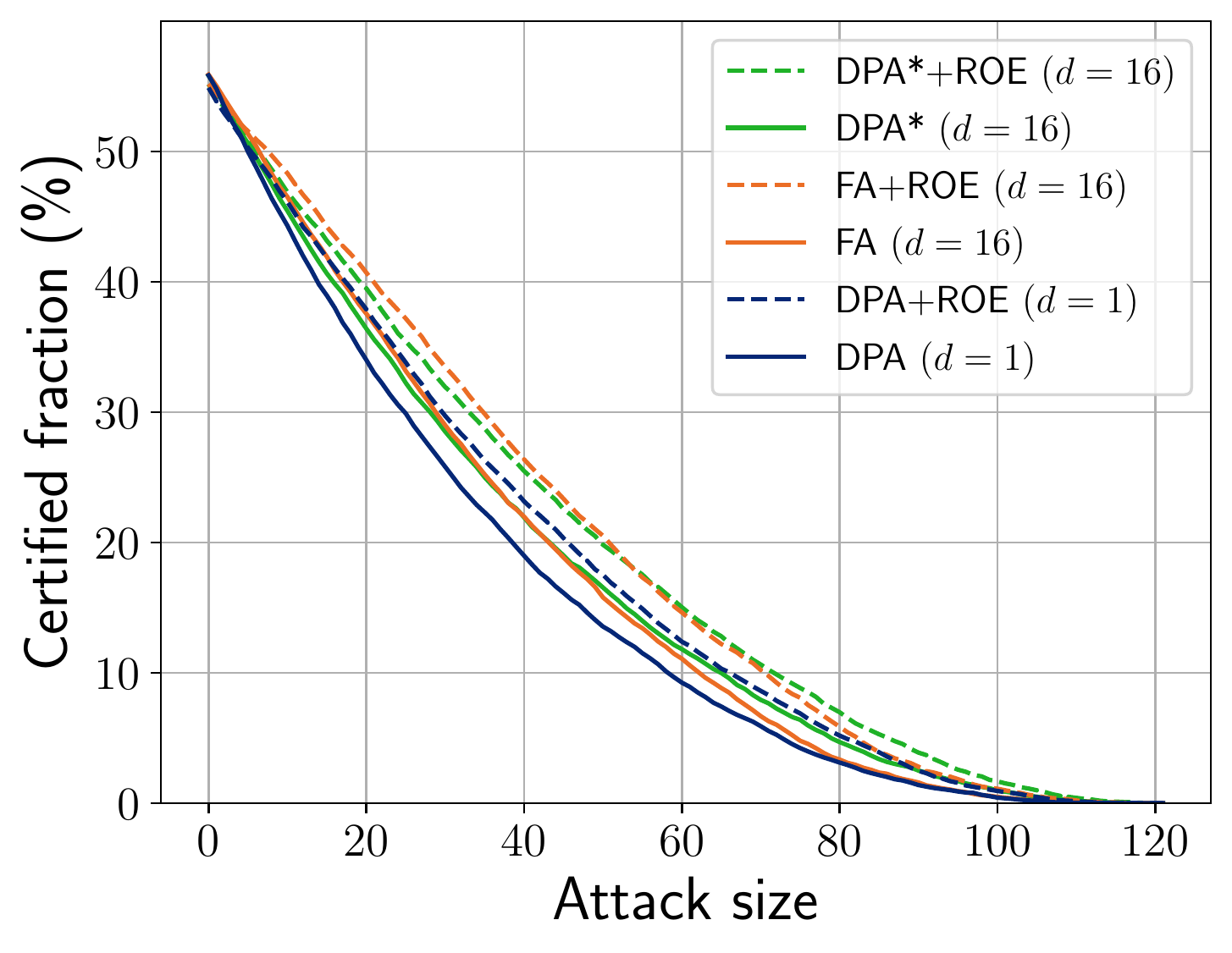}
    \end{minipage}
    \begin{minipage}{0.31\linewidth}
      \includegraphics[width=1\linewidth]{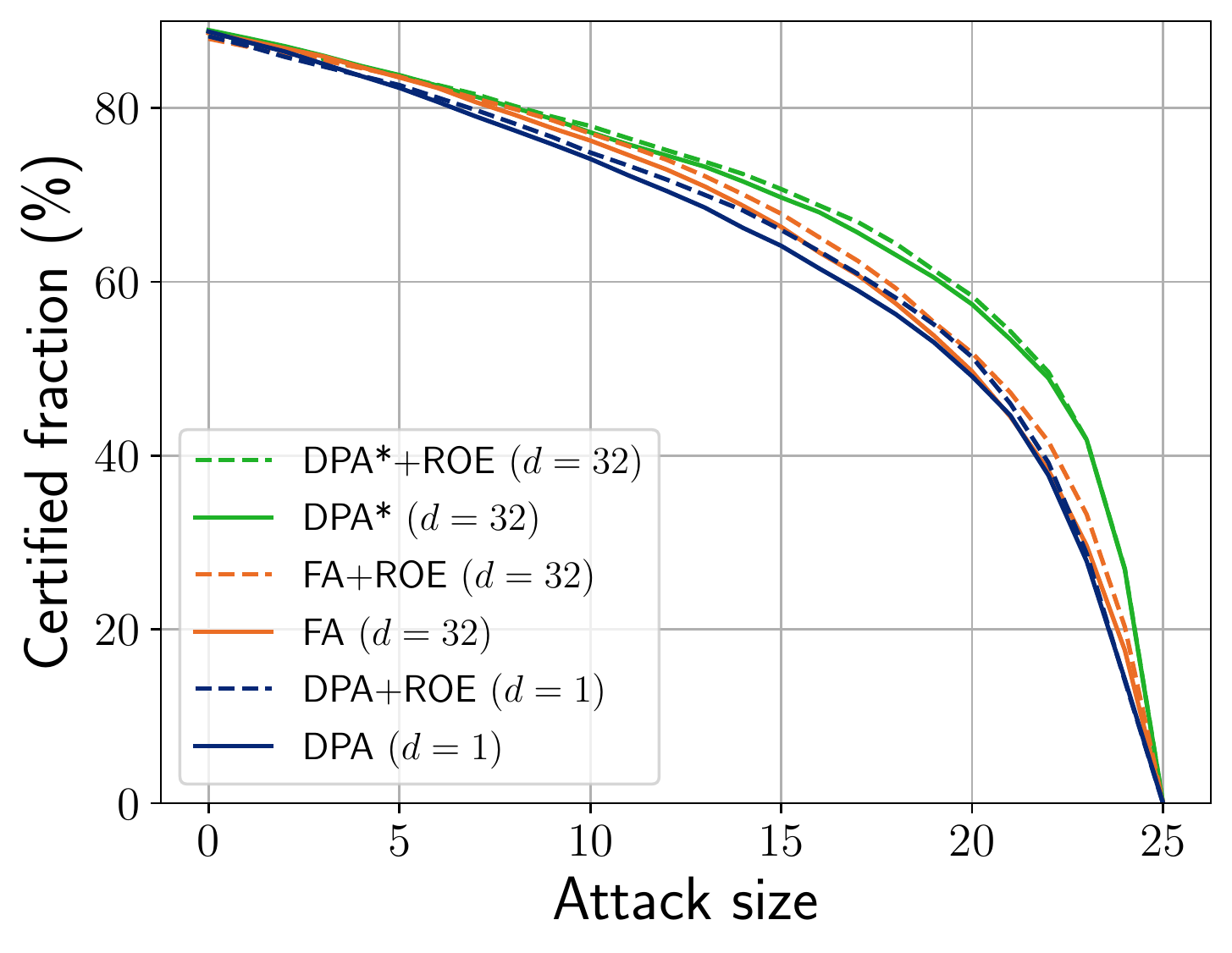}
    \end{minipage}
    \begin{minipage}{0.31\linewidth}
      \includegraphics[width=1\linewidth]{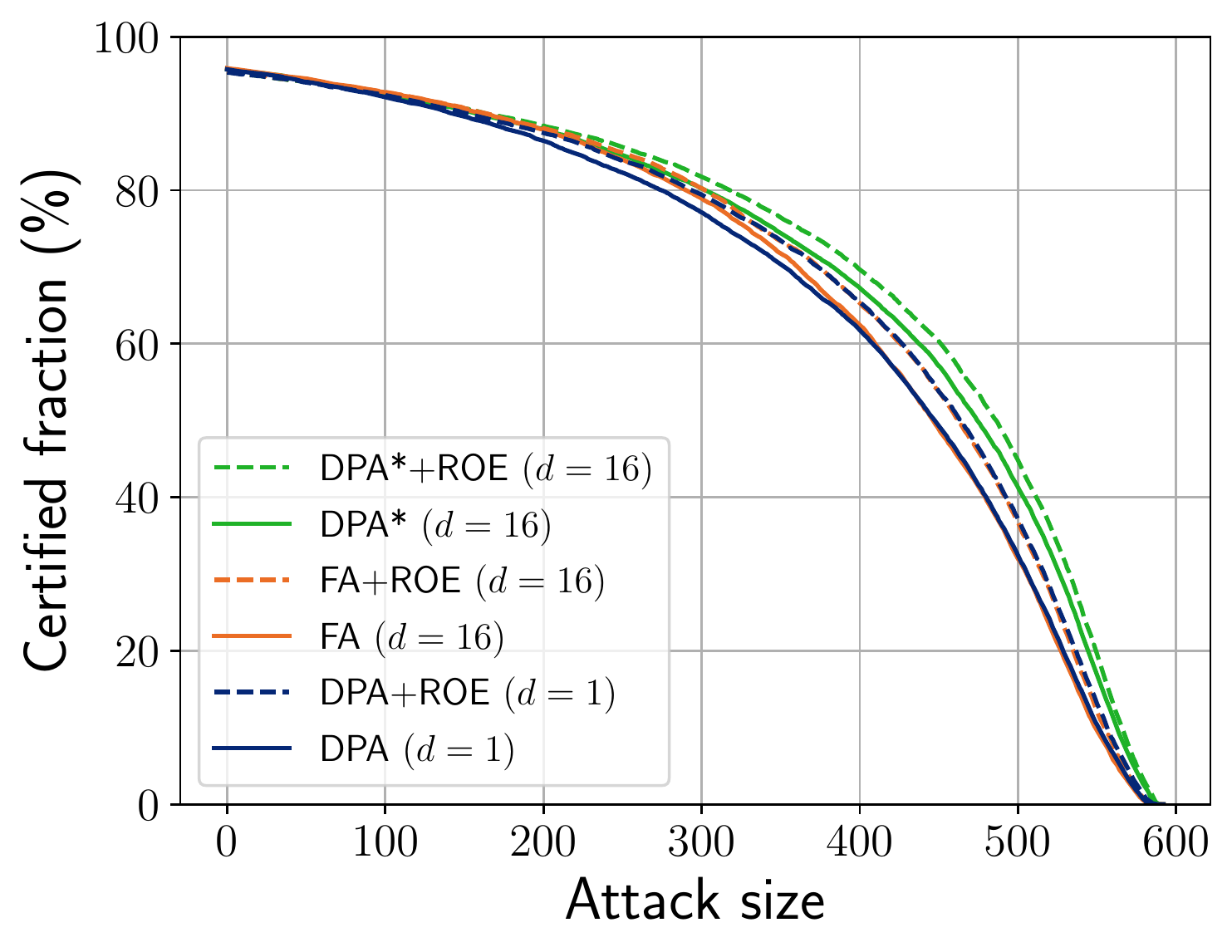}
    \end{minipage}
    \label{fig:certified_accuracy}
  \end{subfigure}\\
  \begin{subfigure}{1\linewidth}
    \centering
    \begin{minipage}{0.31\linewidth}
      \includegraphics[width=1\linewidth]{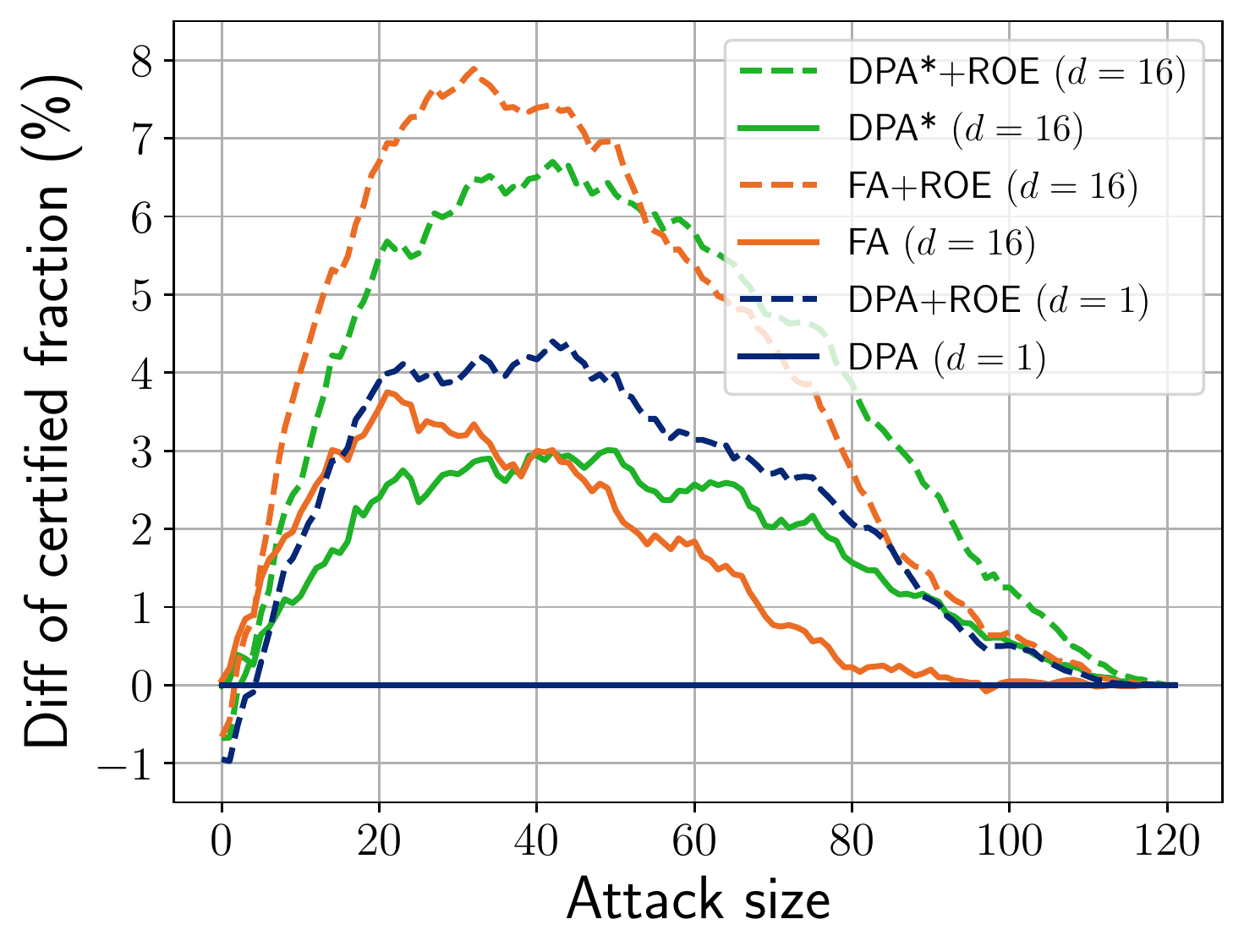}
    \end{minipage}
    \begin{minipage}{0.31\linewidth}
      \includegraphics[width=1\linewidth]{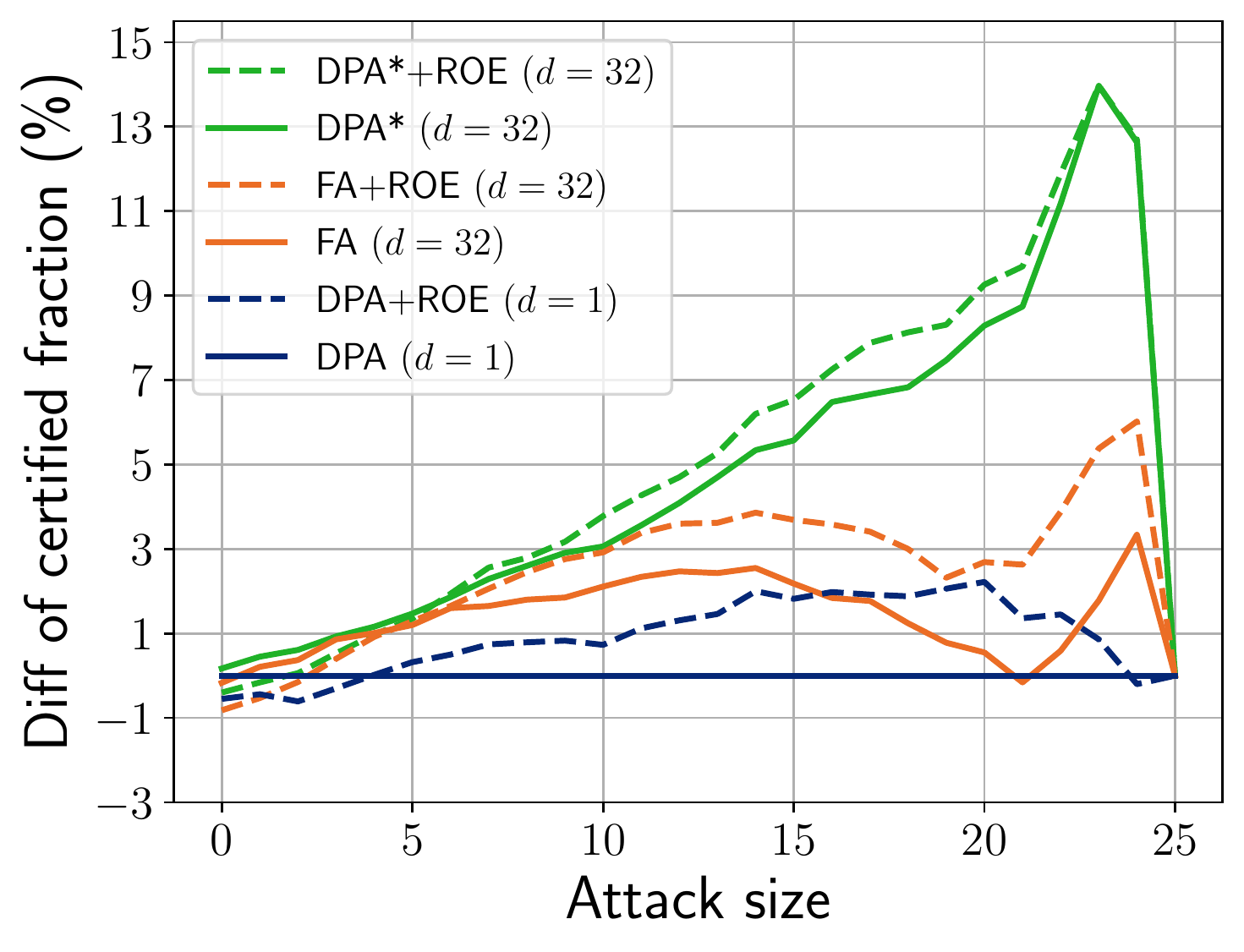}
    \end{minipage}
    \begin{minipage}{0.31\linewidth}
      \includegraphics[width=1\linewidth]{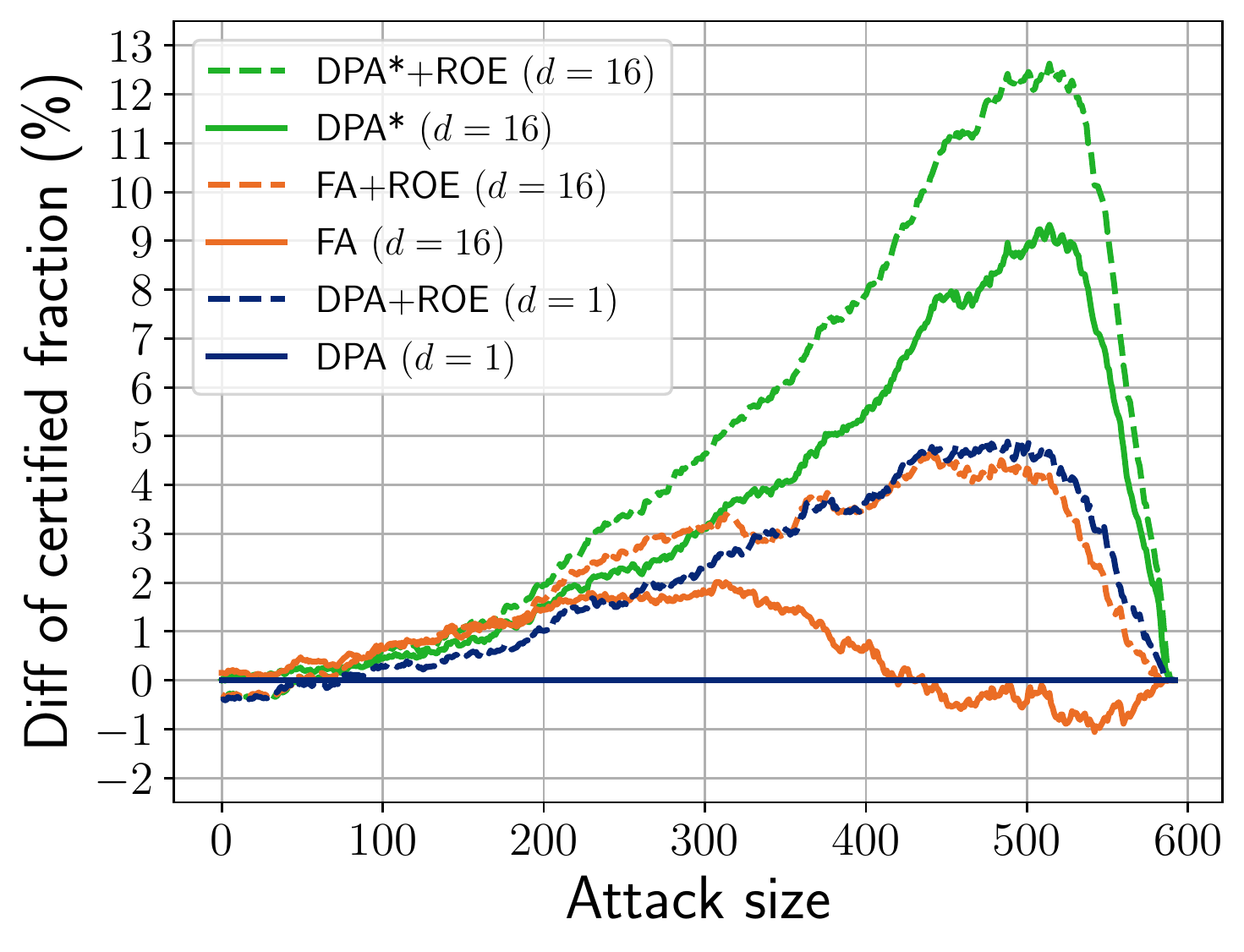}
    \end{minipage}
    \label{fig:improvement}
  \end{subfigure}\\
\caption{
\textbf{First row}:
The curves of certified fraction of different methods on different datasets.
\textbf{Second row}:
The improvements of certified fraction over DPA. 
Plots in the first columns refers to CIFAR-10 ($k=250$),
plots in the second column refers to GRSTB ($k=50$),
and plots in the last column corresponds to MNIST ($k=1200$).
Note that training cost of different methods scales up with $d$, i.e., 
\textbf{training of FA, \faroe{}, \dpastar{}, and \dpastarroe{} with parameter $d$} takes roughly $d$ \textbf{times more than that of DPA or \dparoe{}}.
When the adversary's budget is large, \textbf{\dparoe{} outperforms FA} while it significantly exploits less training cost.
In some case, \dparoe{} can outperform \dpastar{} as well.}
\label{fig:certified_fraction}
\end{figure*}

In this section, we empirically analyze our method and demonstrate that
it reaches state-of-the-art results in certified robustness. 
In some cases, this comes with considerably less computation than the baseline.

\subsection{Experimental setting}

We consider the same setup as prior work \cite{levine2020deep, wang2022improved} and
evaluate our method on MNIST \cite{lecun1998gradient},
CIFAR-10 \cite{krizhevsky2009learning} and GTSRB~\cite{stallkamp2012man} datasets.
We similarly use Network-In-Network \cite{lin2013network} architecture,
to be trained with the set of hyperparameters from \cite{gidaris2018unsupervised}.
\citet{wang2022lethal} observe that the accuracy of ensemble methods can be further improved by having better base classifiers, i.e., base classifiers that have better classification accuracy.
They improve over the original DPA by training base classifiers on the augmented version of datasets.
As we want to have a fair comparison to the FA baseline, we train classifiers of both DPA and FA as \citet{wang2022improved}.

As in prior work, we consider \emph{certified fraction} (CF) as our performance metric.
Given a budget $B$ for the adversary, the certified fraction of the dataset
denotes the fraction of test samples that are \emph{provably} classified correctly
as long as the dataset is not altered by more than $B$ points. 


As baselines, we use the Deep Partition Aggregation (DPA) method of \citet{levine2020deep} and
the Finite Aggregation (FA) method of \citet{wang2022improved}.
As discussed in \citet{wang2022improved}, FA is effectively a generalization of DPA that
uses overlapping partitions. Compared to DPA, FA takes an additional parameter $d$ 
and uses $d$ times as many base models.
When using $d=1$, FA coincides with DPA.
While larger values of $d$ increase the robustness of the model,
this comes at the cost of increased computation;
the training cost for FA is $d$ times the training cost for DPA.

\textbf{Boosted DPA.} Given the increased computational cost of FA, in order to obtain a fair comparison, we also consider a ``boosted'' variant of \dpa{} which we denote by \dpastar{}.\footnote{We thank an anonymous referee of the paper for proposing this method.} For each partition $D_i$ of the dataset, we train $d$ models $\{f_{i, j}\}_{j=1}^{d}$ on the partition using different seed values. At test time, we average the logits layer of these models. In other words, each model $f_i$ in \dpastar is itself an ensemble of $d$ ``submodels''. 
The approach 
effectively makes the predictions more robust to the noisiness of the training process.
The certificate for the model
is calculated using the same technique as the certificate for \dpa{} as outlined in Section \ref{sec:roe_dpa_cert}. We note that the case $d=1$ corresponds to \dpa{}.
\subsection{Results}
\looseness -1
Our results are shown in Tables~\ref{table:roe_fa_vs_fa}, \ref{table:roe_dpas_vs_dpas}, \ref{table:roe_dpa_vs_dpa}, \ref{table:roe_dpa_vs_fa}, \ref{table:roe_dpa_vs_dpa_star}
and Figures~\ref{fig:certified_fraction}, \ref{fig:app_certified_fraction}.
We do not report error bars as the variation caused by training noise is negligible (See Appendix~\ref{app:repeat}).
As seen in Table~\ref{table:roe_fa_vs_fa},
when we apply our aggregation method to FA,
it can remarkably improve the certified accuracy
of the original Finite Aggregation (we compare these two methods with the same $d$). 
Improvements can be up to $3\%$ or $4\%$.
More interestingly, by applying our technique to \dpastar{}, we observe significant improvement which can be up tp $27\%$ in some cases.
This implies that \dpastarroe{} is the new \textbf{state-of-the-art} in provable defense.

Overall, the results show that 
no matter choice of base classifiers is, \textbf{applying ROE improves the certified robustness}.
The results of applying \roe{} to \dpa{} can be seen in Table~\ref{table:roe_dpa_vs_dpa}.
We obtain improvements in certified accuracy by up to $4.73\%, 3.14\%$, and $3.18\%$ on MNIST, CIFAR-10, and GTSRB, respectively.
Improvements of using \roe{} on \dpastar{} is reported in Table~\ref{table:roe_dpas_vs_dpas}.
\dpastarroe{} improves robustness of \dpastar{} on MNIST, CIFAR-10, and GTSRB by up to $3.49\%, 3.70\%$, and $3.44\%$, respectively.

Perhaps more impressively, as seen in Table~\ref{table:roe_dpa_vs_fa},
\dparoe{} also competes with, and for larger values of $B$ outperforms FA
while it significantly needs less training cost as its training cost is equivalent to that of DPA.
For example, on CIFAR-10 dataset when $k=250$,
by using a single NVIDIA GeForce RTX 2080 Ti GPU, the
total training time of classifiers used in \dparoe{} is around $3$ hours while
it takes around $47.3$ hours to train classifiers needed in FA with $d=16$. 
Roughly speaking, the training of FA with parameter $d$, takes $d$ time more than that of DPA, or \dparoe{}.

Although \dparoe{} uses less training time, it obtains a~higher certified accuracy for larger values of $B$, e.g.,
on the standard CIFAR-10 dataset when $k=50$,
it obtains a higher certified fraction than FA with $d=32$ when $B \ge 15$
\textbf{even though it uses 32 times less computation in training}.
A similar comparison
of \dparoe{} and \dpastar{}
in Table~\ref{table:roe_dpa_vs_dpa_star} shows that in some cases, \dparoe{} can outperform \dpastar as well while it uses significantly less training cost.

The increased accuracy of the models depends on the setting; one may choose  \dparoe{}, \faroe{} or \dpastarroe{} as each have their advantages. \dparoe{} has the advantage that it uses less computational resources while \dpastarroe{} and \faroe{} obtain better accuracy. We note that the benefit of \dpastarroe{} and \faroe{} compared to \dparoe{} seem to come from two different sources; \dpastarroe{} uses increased computation to make each model more robust while \faroe{} uses a more complex partitioning scheme. As such, it can be expected that either of the methods may be preferable depending on the setting and
this is observed in our experiments, though 
generally \dpastarroe{} seems to have higher accuracy. Visual comparison of different methods can be seen in Figures~\ref{fig:certified_fraction} and \ref{fig:app_certified_fraction}. While these distinctions are interesting, they are orthogonal to the focus of our paper; we show that using \roe{} improves robustness in all of these settings without increasing the training cost.

\textbf{Effect of the budget $B$.}
Our results show that ROE methods are especially useful for larger values of 
the adversary's budget $B$.
Intuitively, \roe{} is utilizing base classifiers that were previously discarded. As such, for a fixed budget $B$, \emph{the ratio} of the poisoned
samples to the utilized models is considerably smaller for our method,
which allows us to obtain improved results.
We note that this is in strong contrast to FA, where for larger values of $B$,
the accuracy gains compared to DPA diminish and eventually cease to exist.
Indeed, as seen in Figure~\ref{fig:certified_fraction},
FA can actually be \emph{worse} than DPA for large budgets,
while our method remains strongly favorable, as we achieve 5\% higher certified fraction
on the standard CIFAR-10 dataset.

While our aggregation method performs well when the adversary's budget is high,
we see a slightly lower certified fraction
in
Figures~\ref{fig:certified_fraction} and \ref{fig:app_certified_fraction}
when $B$ is relatively small.
In these cases, the certified fraction is close to clean accuracy, i.e.,
accuracy of the model when training data is not poisoned. 
Intuitively, the drop is because of
the (slightly) lower reliability of the logits-layer information as discussed in Section \ref{sec:roe}.
\roe{} methods have slightly lower clean accuracy because they involve all models in prediction,
even models which are not accurate for a sample test.
On the other hand, when the model's prediction is correct,
involving more models makes the adversary's task harder.

\section{Conclusion}
\looseness -1
In this paper, we introduced
Run-Off Election (\roe{}), 
a new aggregation method for ensemble-based defenses against data poisoning.
We proposed a novel two-stage election
across the base models of the ensemble that utilizes all of the models in order to increase the prediction gap between the top and runner-up classes.
We developed a unified
framework for calculating
certificates for our method
and 
proposed three new defense methods -- \dparoe{}, \faroe{}, 
and \dpastarroe{} -- based
on prior work and a new \dpastar{} method proposed by an anonymous reviewer.
We evaluated our methods on standard
benchmarks
and
observed improved poisoning certificates while simultaneously reducing the training cost. Our method established a new state-of-the-art in provable defense against general data poisoning in several datasets. 

\looseness -1
For future work, it would be interesting to extend our methodology to other ensemble methods including the ones producing stochastic certificates. 
As discussed in Section~\ref{sec:method}, in principle, \roe{} can be applied on top of any ensemble method, though it is not immediately clear how one can obtain prediction certificates for such hybrid models.
We hope that our unified approach to calculating
certificates in Section \ref{sec:roe_dpa_cert},
can be helpful for other ensemble methods as well.

Another interesting direction is
to
explore more complex aggregation mechanisms, such as an $N$-round election for $N > 2$.
Two potential challenges for designing such methods are certificate calculation and potential decreased accuracy due to unreliable logits-layer information.
These challenges also appear in our work;
our method for calculating certificates is more involved than prior work and in some settings our models have lower clean accuracy.
We hope that the techniques proposed in our paper can help address these challenges for more complex mechanisms as well.

\section{Acknowledgements}
The authors thank Wenxiao Wang for helpful discussions throughout the project. 
This project was supported in part by NSF CAREER AWARD 1942230, a grant from NIST 60NANB20D134, HR001119S0026 (GARD), ONR YIP award N00014-22-1-2271, Army Grant No. W911NF2120076 and the NSF award CCF2212458.

\bibliography{references}
\bibliographystyle{icml2023}


\appendix
\onecolumn

\appendix
\onecolumn

\section{Code}
Our code can be found in this \href{https://github.com/k1rezaei/Run-Off-Election/tree/main}{github repository}.

\section{Figures and Tables}
In this section, we provide Tables~\ref{table:roe_dpa_vs_dpa}, \ref{table:roe_dpa_vs_fa}, and \ref{table:roe_dpa_vs_dpa_star}.
Figure~\ref{fig:app_certified_fraction} is also depicted here.


\begin{table*}[tb]
    \caption{
        Certified fraction of \dpastarroe{}{}, and \dpastar{} with various values for hyperparameter $d$, with respect to different attack sizes $B$.
        Improvements over the \dpastar{} baseline are highlighted in blue if they are positive and red otherwise.
    }
    \vskip 0.15in
    \label{table:roe_dpas_vs_dpas}
    \resizebox{\linewidth}{!}{
    \begin{tabular}{|c||c||c|c||c|c|c|c|c|}
    \hline
        dataset & $k$ & method & $d$ & \multicolumn{5}{|c|}{certified fraction}\\ 
        \hline\hline
        \multirow{5}{*}{MNIST}      & \multirow{3}{*}{1200}    & \multicolumn{2}{c||}{}                    & $B\leq 100$ & $B\leq 200$ & $B\leq 300$ & $B\leq 400$ & $B\leq 500$\\\cline{3-9}
                                    &                          & \footnotesize{\dpastar{}}             & \multirow{2}{*}{16}        & 92.55\% & 87.99\% & 80.23\% & 67.25\% & 41.32\%\\\cline{5-9}
                                    &                          & \footnotesize{\dpastarroe{}}       &         & 92.70\%{\color{blue}(+0.15\%)} & 88.41\%{\color{blue}(+0.42\%)} & 81.75\%{\color{blue}(+1.52\%)} & 69.67\%{\color{blue}(+2.42\%)} & 44.81\%{\color{blue}(+3.49\%)} \\\cline{3-9}
        \hline\hline
        \multirow{10}{*}{CIFAR-10}  & \multirow{5}{*}{50}      & \multicolumn{2}{c||}{}                    & $B\leq 5$ & $B\leq 10$ & $B\leq 15$ & $B\leq 18$ & $B\leq 20$\\\cline{3-9}
                                    &                          & \footnotesize{\dpastar{}}             & \multirow{2}{*}{16}        & 61.00\% & 50.94\% & 39.29\% & 31.41\% & 25.97\%\\\cline{5-9}
                                    &                          & \footnotesize{\dpastarroe{}}       &         & 61.87\%{\color{blue}(+0.87\%)} & 52.71\%{\color{blue}(+1.77\%)} & 41.51\%{\color{blue}(+2.22\%)} & 33.42\%{\color{blue}(+2.01\%)} & 27.47\%{\color{blue}(+1.5\%)}\\\cline{3-9}
                                    &                          & \footnotesize{\dpastar{}}             & \multirow{2}{*}{32}        & 65.77\% & 60.89\% & 55.53\% & 51.05\% & 46.95\%\\\cline{5-9}
                                    &                          & \footnotesize{\dpastarroe{}}       &         & 65.99\%{\color{blue}(+0.22\%)} & 61.51\%{\color{blue}(+0.62\%)} & 56.13\%{\color{blue}(+0.6\%)} & 51.83\%{\color{blue}(+0.78\%)} & 47.61\%{\color{blue}(+0.66\%)}\\\cline{3-9}
                                    & \multirow{5}{*}{250}     & \multicolumn{2}{c||}{}                    & $B\leq 10$ & $B\leq 20$ & $B\leq 40$ & $B\leq 50$ & $B\leq 60$\\\cline{3-9}
                                    &                          & \footnotesize{\dpastar{}}             & \multirow{2}{*}{8}         & 45.36\% & 36.34\% & 21.71\% & 16.41\% & 11.60\%\\\cline{5-9}
                                    &                          & \footnotesize{\dpastarroe{}}       &          & 47.14\%{\color{blue}(+1.78\%)} & 39.32\%{\color{blue}(+2.98\%)} & 25.41\%{\color{blue}(+3.7\%)} & 19.68\%{\color{blue}(+3.27\%)} & 15.02\%{\color{blue}(+3.42\%)}\\\cline{3-9}
                                    &                          & \footnotesize{\dpastar{ }}             & \multirow{2}{*}{16}        & 45.45\% & 36.41\% & 21.93\% & 16.55\% & 11.82\%\\\cline{5-9}
                                    &                          & \footnotesize{\dpastarroe{}}       &         & 46.88\%{\color{blue}(+1.43\%)} & 39.50\%{\color{blue}(+3.09\%)} & 25.49\%{\color{blue}(+3.56\%)} & 19.83\%{\color{blue}(+3.28\%)} & 15.04\%{\color{blue}(+3.22\%)}\\\cline{3-9}
        \hline\hline
        \multirow{10}{*}{GTSRB}     & \multirow{5}{*}{50}      & \multicolumn{2}{c||}{}                    & $B\leq 5$ & $B\leq 10$ & $B\leq 15$ & $B\leq 20$ & $B\leq 22$\\\cline{3-9}
                                    &                          & \footnotesize{\dpastar{}}             & \multirow{2}{*}{16}        & 83.71\% & 77.22\% & 69.70\% & 56.71\% & 48.61\%\\\cline{5-9}
                                    &                          & \footnotesize{\dpastarroe{}}       &         & 83.67\%{\color{red}(--0.04\%)} & 77.84\%{\color{blue}(+0.62\%)} & 70.63\%{\color{blue}(+0.93\%)} & 57.97\%{\color{blue}(+1.26\%)} & 49.33\%{\color{blue}(+0.72\%)}\\\cline{3-9}
                                    
                                    &                          & \footnotesize{\dpastar{}}             & \multirow{2}{*}{32}        & 83.79\% & 77.21\% & 69.71\% & 57.41\% & 48.91\%\\\cline{5-9}
                                    &                          & \footnotesize{\dpastarroe{}}       &         & 83.67\%{\color{red}(--0.13\%)} & 77.93\%{\color{blue}(+0.71\%)} & 70.67\%{\color{blue}(+0.97\%)} & 58.38\%{\color{blue}(+0.97\%)} & 49.61\%{\color{blue}(+0.7\%)}\\\cline{2-9}
                                    & \multirow{5}{*}{100}     & \multicolumn{2}{c||}{}                    & $B\leq 5$ & $B\leq 15$ & $B\leq 20$ & $B\leq 25$ & $B\leq 30$\\\cline{3-9}
                                    &                          & \footnotesize{\dpastar{}}             & \multirow{2}{*}{16}        & 47.64\% & 34.73\% & 27.50\% & 21.20\% & 16.19\%\\\cline{5-9}
                                    &                          & \footnotesize{\dpastarroe{}}       &         & 47.50\%{\color{red}(--0.14\%)} & 36.48\%{\color{blue}(+1.74\%)} & 30.60\%{\color{blue}(+3.1\%)} & 24.63\%{\color{blue}(+3.44\%)} & 19.26\%{\color{blue}(+3.07\%)}\\\cline{3-9}
                                    &                          & \footnotesize{\dpastar{}}             & \multirow{2}{*}{32}        & 47.82\% & 34.81\% & 27.90\% & 21.51\% & 16.34\%\\\cline{5-9}
                                    &                          & \footnotesize{\dpastarroe{}}       &         & 47.66\%{\color{red}(--0.17\%)} & 36.66\%{\color{blue}(+1.84\%)} & 30.70\%{\color{blue}(+2.8\%)} & 24.71\%{\color{blue}(+3.2\%)} & 19.33\%{\color{blue}(+2.99\%)}\\\cline{3-9}                                
        \hline
    \end{tabular}}
\end{table*}

\begin{table}[tb]
    \caption{
        Certified fraction of \dparoe{}, and original DPA with respect to different attack sizes $B$.
        Improvements over the DPA baseline are highlighted in blue if they are positive and red otherwise.
    }
    \vskip 0.15in
    \label{table:roe_dpa_vs_dpa}
    \resizebox{\linewidth}{!}{
    \begin{tabular}{|c||c||c||c|c|c|c|c|}
    \hline
        dataset & $k$ & method & \multicolumn{5}{|c|}{certified fraction}\\ 
        \hline\hline
        \multirow{3}{*}{MNIST}      & \multirow{3}{*}{1200}    &                                      & $B\leq 100$ & $B\leq 200$ & $B\leq 300$ & $B\leq 400$ & $B\leq 500$\\\cline{3-8}
                                    &                          & \footnotesize{DPA}                   & 92.11\% & 86.45\% & 77.12\% & 61.78\% & 32.42\%\\\cline{3-8}
                                    &                          & \footnotesize{\dparoe{}}             & 92.38\%{\color{blue}(+0.27\%)} & 87.46\%{\color{blue}(+1.01\%)} & 79.43\%{\color{blue}(+2.31\%)} & 65.42\%{\color{blue}(+3.64\%)} & 37.15\%{\color{blue}(+4.73\%)}\\\cline{2-8}
        \hline\hline
        \multirow{6}{*}{CIFAR-10}   & \multirow{3}{*}{50}      &                                      & $B\leq 5$ & $B\leq 10$ & $B\leq 15$ & $B\leq 18$ & $B\leq 20$\\\cline{3-8}
                                    &                          & \footnotesize{DPA}                   & 58.07\% & 46.44\% & 33.46\% & 24.87\% & 19.36\% \\\cline{3-8}
                                    &                          & \footnotesize{\dparoe{}}             & 59.80\%{\color{blue}(+1.73\%)} & 49.09\%{\color{blue}(+2.65\%)} & 36.04\%{\color{blue}(+2.58\%)} & 27.52\%{\color{blue}(+2.65\%)} & 21.30\%{\color{blue}(+1.94\%)}\\\cline{2-8}
                                    & \multirow{3}{*}{250}     &                                      & $B\leq 10$ & $B\leq 20$ & $B\leq 40$ & $B\leq 50$ & $B\leq 60$\\\cline{3-8}
                                    &                          & \footnotesize{DPA}                   & 44.31\% & 34.01\% & 18.99\% & 13.55\% & 9.25\%\\\cline{3-8}
                                    &                          & \footnotesize{\dparoe{}}             & 46.14\%{\color{blue}(+1.83\%)} & 37.90\%{\color{blue}(+3.89\%)} & 23.16\%{\color{blue}(+4.17\%)} & 17.53\%{\color{blue}(+3.98\%)} & 12.39\%{\color{blue}(+3.14\%)}\\\cline{2-8}
        \hline\hline
        \multirow{6}{*}{GTSRB}      & \multirow{3}{*}{50}      &                                      & $B\leq 5$ & $B\leq 10$ & $B\leq 15$ & $B\leq 20$ & $B\leq 22$\\\cline{3-8}
                                    &                          & \footnotesize{DPA}                   & 82.32\% & 74.15\% & 64.14\% & 49.12\% & 37.73\% \\\cline{3-8}
                                    &                          & \footnotesize{\dparoe{}}             & 82.64\%{\color{blue}(+0.32\%)} & 74.88\%{\color{blue}(+0.73\%)} & 65.96\%{\color{blue}(+1.82\%)} & 51.34\%{\color{blue}(+2.22\%)} & 39.18\%{\color{blue}(+1.45\%)}\\\cline{2-8}
                                    & \multirow{3}{*}{100}     &                                      & $B\leq 5$ & $B\leq 15$ & $B\leq 20$ & $B\leq 25$ & $B\leq 30$\\\cline{3-8}
                                    &                          & \footnotesize{DPA}                   & 46.16\% & 30.19\% & 22.84\% & 17.16\% & 12.75\%\\\cline{3-8}
                                    &                          & \footnotesize{\dparoe{}}             & 46.09\%{\color{red}(--0.07\%)} & 33.45\%{\color{blue}(+3.26\%)} & 26.86\%{\color{blue}(+4.02\%)} & 21.02\%{\color{blue}(+3.86\%)} & 15.93\%{\color{blue}(+3.18\%)}\\\cline{2-8}
        \hline
    \end{tabular}}
\end{table}
\begin{table}[H]
    \caption{
        Certified fraction of \dparoe{}, and original FA with various values of hyperparameter $d$
        with respect to different attack sizes $B$.
        Improvements of \dparoe{} compared to the original FA with different values of $d$ are highlighted in blue if they are positive and red otherwise.
        \textbf{Note that FA with parameter $d$ uses $d$ times as many as classifiers than \dparoe{}. Training FA classifiers therefore takes $d$ times more that that of \dparoe{}.}
    }
    \vskip 0.15in
    \label{table:roe_dpa_vs_fa}
    \resizebox{\linewidth}{!}{ 
    \begin{tabular}{|c||c||c|c||c|c|c|c|c|}
    \hline
        dataset & $k$ & method & $d$ & \multicolumn{5}{|c|}{certified fraction}\\ 
        \hline\hline
        \multirow{5}{*}{MNIST}      & \multirow{5}{*}{1200}    & \multicolumn{2}{c||}{}                            & $B\leq 100$ & $B\leq 200$ & $B\leq 300$ & $B\leq 400$ & $B\leq 500$\\\cline{3-9}
                                    &                          & \footnotesize{FA}              & 16               & 92.75\% & 87.89\% & 78.91\% & 62.42\% & 31.97\%\\\cline{3-9}
                                    &                          & \multicolumn{2}{c||}{\footnotesize{\dparoe{}}}    & 92.38\%{\color{red}(--0.37\%)} & 87.46\%{\color{red}(--0.43\%)} & 79.43\%{\color{blue}(+0.52\%)} & 65.42\%{\color{blue}(+3.00\%)} & 37.15\%{\color{blue}(+5.18\%)} \\\cline{3-9}
                                    &                          & \footnotesize{FA}              & 32               & 92.97\% & 88.49\% & 80.17\% & 64.34\% & 31.09\%\\\cline{3-9}
                                    &                          & \multicolumn{2}{c||}{\footnotesize{\dparoe{}}}    & 92.38\%{\color{red}(--0.59\%)} & 87.46\%{\color{red}(--1.03\%)} & 79.43\%{\color{red}(--0.74\%)} & 65.42\%{\color{blue}(+1.08\%)} & 37.15\%{\color{blue}(+6.06\%)}\\\cline{3-9}
        \hline\hline
        \multirow{10}{*}{CIFAR-10}   & \multirow{5}{*}{50}     & \multicolumn{2}{c||}{}                            & $B\leq 5$ & $B\leq 10$ & $B\leq 15$ & $B\leq 18$ & $B\leq 20$\\\cline{3-9}
                                    &                          & \footnotesize{FA}             & 16                & 60.55\% & 48.85\% & 34.61\% & 25.46\% & 19.90\%\\\cline{3-9}
                                    &                          & \multicolumn{2}{c||}{\footnotesize{\dparoe{}}}    & 59.80\%{\color{red}(--0.75\%)} & 49.09\%{\color{blue}(+0.24\%)} & 36.04\%{\color{blue}(+1.43\%)} & 27.52\%{\color{blue}(+2.06\%)} & 21.30\%{\color{blue}(+1.40\%)}\\\cline{3-9}
                                    &                          &\footnotesize{FA}             & 32                 & 61.31\% & 50.31\% & 36.03\% & 26.55\% & 19.93\%\\\cline{3-9}
                                    &                          & \multicolumn{2}{c||}{\footnotesize{\dparoe{}}}    & 59.80\%{\color{red}(--1.51\%)} & 49.09\%{\color{red}(--1.22\%)} & 36.04\%{\color{blue}(+0.01\%)} & 27.52\%{\color{blue}(+0.97\%)} & 21.30\%{\color{blue}(+1.37\%)}\\\cline{2-9}
                                    & \multirow{5}{*}{250}     & \multicolumn{2}{c||}{}                            & $B\leq 10$ & $B\leq 20$ & $B\leq 40$ & $B\leq 50$ & $B\leq 60$\\\cline{3-9}
                                    &                          & \footnotesize{FA}             & 8                 & 45.38\% & 36.05\% & 20.08\% & 14.39\% & 9.70\%\\\cline{3-9}
                                    &                          & \multicolumn{2}{c||}{\footnotesize{\dparoe{}}}    & 46.14\%{\color{blue}(+0.76\%)} & 37.90\%{\color{blue}(+1.85\%)} & 23.16\%{\color{blue}(+3.08\%)} & 17.53\%{\color{blue}(+3.14\%)} & 12.39\%{\color{blue}(+2.69\%)}\\\cline{3-9}
                                    &                          & \footnotesize{FA}             & 16                & 46.52\% & 37.56\% & 21.99\% & 15.79\% & 11.09\%\\\cline{3-9}
                                    &                          & \multicolumn{2}{c||}{\footnotesize{\dparoe{}}}    & 46.14\%{\color{red}(--0.38\%)} & 37.90\%{\color{blue}(+0.34\%)} & 23.16\%{\color{blue}(+1.17\%)} & 17.53\%{\color{blue}(+1.74\%)} & 12.39\%{\color{blue}(+1.30\%)}\\\cline{3-9}
        \hline\hline
        \multirow{10}{*}{GTSRB}     & \multirow{5}{*}{50}      & \multicolumn{2}{c||}{}                            & $B\leq 5$ & $B\leq 10$ & $B\leq 15$ & $B\leq 20$ & $B\leq 22$\\\cline{3-9}
                                    &                          & \footnotesize{FA}             & 16                & 82.71\% & 74.66\% & 63.77\% & 47.52\% & 35.54\%\\\cline{3-9}
                                    &                          & \multicolumn{2}{c||}{\footnotesize{\dparoe{}}}    & 82.64\%{\color{red}(--0.07\%)} & 74.88\%{\color{blue}(+0.22\%)} & 65.96\%{\color{blue}(+2.19\%)} & 51.34\%{\color{blue}(+3.82\%)} & 39.18\%{\color{blue}(+3.63\%)}\\\cline{3-9}
                                    &                          & \footnotesize{FA}             & 32                & 83.52\% & 76.26\% & 66.32\% & 49.68\% & 38.31\%\\\cline{3-9}
                                    &                          & \multicolumn{2}{c||}{\footnotesize{\dparoe{}}}    & 82.64\%{\color{red}(--0.88\%)} & 74.88\%{\color{red}(--1.39\%)} & 65.96\%{\color{red}(--0.36\%)} & 51.34\%{\color{blue}(+1.66\%)} & 39.18\%{\color{blue}(+0.86\%)}\\\cline{2-9}
                                    & \multirow{5}{*}{100}     & \multicolumn{2}{c||}{}                            & $B\leq 5$ & $B\leq 15$ & $B\leq 20$ & $B\leq 25$ & $B\leq 30$\\\cline{3-9}
                                    &                          & \footnotesize{FA}             & 16                & 48.19\% & 33.95\% & 25.96\% & 18.92\% & 13.82\%\\\cline{3-9}
                                    &                          & \multicolumn{2}{c||}{\footnotesize{\dparoe{}}}    & 46.09\%{\color{red}(--2.10\%)} & 33.45\%{\color{red}(--0.50\%)} & 26.86\%{\color{blue}(+0.90\%)} & 21.02\%{\color{blue}(+2.10\%)} & 15.93\%{\color{blue}(+2.11\%)}\\\cline{3-9}
                                    &                          & \footnotesize{FA}             & 32                & 48.39\% & 34.96\% & 27.05\% & 19.83\% & 14.47\%\\\cline{3-9}
                                    &                          & \multicolumn{2}{c||}{\footnotesize{\dparoe{}}}    & 46.09\%{\color{red}(--2.30\%)} & 33.45\%{\color{red}(--1.51\%)} & 26.86\%{\color{red}(--0.18\%)} & 21.02\%{\color{blue}(+1.19\%)} & 15.93\%{\color{blue}(+1.46\%)}\\\cline{3-9}
        \hline        
    \end{tabular}}
\end{table}

\begin{table}[H]
    \caption{
        Certified fraction of \dparoe{}, and \dpastar{} with various values of hyperparameter $d$
        with respect to different attack sizes $B$.
        Improvements of \dparoe{} compared to the \dpastar{} with different values of $d$ are highlighted in blue if they are positive and red otherwise.
        \textbf{Note that \dpastar{} with parameter $d$ uses $d$ times as many as classifiers than \dparoe{}. Training \dpastar{} classifiers therefore takes $d$ times more that that of \dparoe{}.}
    }
    \vskip 0.15in
    \label{table:roe_dpa_vs_dpa_star}
    \resizebox{\linewidth}{!}{ 
    \begin{tabular}{|c||c||c|c||c|c|c|c|c|}
    \hline
        dataset & $k$ & method & $d$ & \multicolumn{5}{|c|}{certified fraction}\\ 
        \hline\hline
        \multirow{5}{*}{MNIST}      & \multirow{3}{*}{1200}    & \multicolumn{2}{c||}{}                            & $B\leq 100$ & $B\leq 200$ & $B\leq 300$ & $B\leq 400$ & $B\leq 500$\\\cline{3-9}
                                    &                          & \footnotesize{\dpastar{}}              & 16               & 92.55\% & 87.99\% & 80.23\% & 67.25\% & 41.32\%\\\cline{3-9}
                                    &                          & \multicolumn{2}{c||}{\footnotesize{\dparoe{}}}    & 92.38\%{\color{red}(--0.17\%)} & 87.46\%{\color{red}(--0.53\%)} & 79.43\%{\color{red}(--0.80\%)} & 65.42\%{\color{red}(-1.83\%)} & 37.15\%{\color{red}(--4.17\%)} \\\cline{3-9}
        \hline\hline
        \multirow{10}{*}{CIFAR-10}   & \multirow{5}{*}{50}     & \multicolumn{2}{c||}{}                            & $B\leq 5$ & $B\leq 10$ & $B\leq 15$ & $B\leq 18$ & $B\leq 20$\\\cline{3-9}
                                    &                          & \footnotesize{\dpastar{}}             & 16                & 61.00\% & 50.94\% & 39.29\% & 31.41\% & 25.97\%\\\cline{3-9}
                                    &                          & \multicolumn{2}{c||}{\footnotesize{\dparoe{}}}    & 59.80\%{\color{red}(--1.20\%)} & 49.09\%{\color{red}(--1.85\%)} & 36.04\%{\color{red}(--3.25\%)} & 27.52\%{\color{red}(--3.89\%)} & 21.30\%{\color{red}(--4.67\%)}\\\cline{3-9}
                                    &                          &\footnotesize{\dpastar{}}             & 32                 & 65.77\% & 60.89\% & 55.53\% & 51.05\% & 46.95\%\\\cline{3-9}
                                    &                          & \multicolumn{2}{c||}{\footnotesize{\dparoe{}}}    & 59.80\%{\color{red}(--5.97\%)} & 49.09\%{\color{red}(--11.80\%)} & 36.04\%{\color{red}(--19.49\%)} & 27.52\%{\color{red}(--23.53\%)} & 21.30\%{\color{red}(--25.65\%)}\\\cline{2-9}
                                    & \multirow{5}{*}{250}     & \multicolumn{2}{c||}{}                            & $B\leq 10$ & $B\leq 20$ & $B\leq 40$ & $B\leq 50$ & $B\leq 60$\\\cline{3-9}
                                    &                          & \footnotesize{\dpastar{}}             & 8                 & 45.36\% & 36.34\% & 21.71\% & 16.41\% & 11.60\% \\\cline{3-9}
                                    &                          & \multicolumn{2}{c||}{\footnotesize{\dparoe{}}}    & 46.14\%{\color{blue}(+0.78\%)} & 37.90\%{\color{blue}(+1.56\%)} & 23.16\%{\color{blue}(+1.45\%)} & 17.53\%{\color{blue}(+1.12\%)} & 12.39\%{\color{blue}(+0.79\%)}\\\cline{3-9}
                                    &                          & \footnotesize{\dpastar{}}             & 16                & 45.45\% & 36.41\% & 21.93\% & 16.55\% & 11.82\%\\\cline{3-9}
                                    &                          & \multicolumn{2}{c||}{\footnotesize{\dparoe{}}}    & 46.14\%{\color{blue}(+0.69\%)} & 37.90\%{\color{blue}(+1.49\%)} & 23.16\%{\color{blue}(+1.23\%)} & 17.53\%{\color{blue}(+0.98\%)} & 12.39\%{\color{blue}(+0.57\%)}\\\cline{3-9}
        \hline\hline
        \multirow{10}{*}{GTSRB}     & \multirow{5}{*}{50}      & \multicolumn{2}{c||}{}                            & $B\leq 5$ & $B\leq 10$ & $B\leq 15$ & $B\leq 20$ & $B\leq 22$\\\cline{3-9}
                                    &                          & \footnotesize{\dpastar{}}             & 16                & 83.71\% & 77.22\% & 69.70\% & 56.71\% & 48.61\%\\\cline{3-9}
                                    &                          & \multicolumn{2}{c||}{\footnotesize{\dparoe{}}}    & 82.64\%{\color{red}(--1.07\%)} & 74.88\%{\color{red}(--2.34\%)} & 65.96\%{\color{red}(--3.74\%)} & 51.34\%{\color{red}(--5.38\%)} & 39.18\%{\color{red}(--9.44\%)}\\\cline{3-9}
                                    &                          & \footnotesize{\dpastar{}}             & 32                & 83.79\% & 77.21\% & 69.71\% & 57.41\% & 48.91\%\\\cline{3-9}
                                    &                          & \multicolumn{2}{c||}{\footnotesize{\dparoe{}}}    & 82.64\%{\color{red}(--1.16\%)} & 74.88\%{\color{red}(--2.34\%)} & 65.96\%{\color{red}(--3.75\%)} & 51.34\%{\color{red}(--6.07\%)} & 39.18\%{\color{red}(--9.73\%)}\\\cline{2-9}
                                    & \multirow{5}{*}{100}     & \multicolumn{2}{c||}{}                            & $B\leq 5$ & $B\leq 15$ & $B\leq 20$ & $B\leq 25$ & $B\leq 30$\\\cline{3-9}
                                    &                          & \footnotesize{\dpastar{}}             & 16                & 47.64\% & 34.73\% & 27.50\% & 21.20\% & 16.19\%\\\cline{3-9}
                                    &                          & \multicolumn{2}{c||}{\footnotesize{\dparoe{}}}    & 46.09\%{\color{red}(--1.55\%)} & 33.45\%{\color{red}(--1.28\%)} & 26.86\%{\color{red}(--0.63\%)} & 21.02\%{\color{red}(--0.17\%)} & 15.93\%{\color{red}(--0.26\%)}\\\cline{3-9}
                                    &                          & \footnotesize{\dpastar{}}             & 32                & 47.82\% & 34.81\% & 27.90\% & 21.51\% & 16.34\%\\\cline{3-9}
                                    &                          & \multicolumn{2}{c||}{\footnotesize{\dparoe{}}}    & 46.09\%{\color{red}(--1.73\%)} & 33.45\%{\color{red}(--1.36\%)} & 26.86\%{\color{red}(--1.04\%)} & 21.02\%{\color{red}(--0.49\%)} & 15.93\%{\color{red}(--0.41\%)}\\\cline{3-9}
        \hline        
    \end{tabular}}
\end{table}

\begin{figure*}[tbp]
  \centering
  \begin{subfigure}{1\linewidth}
    \centering
    \begin{minipage}{0.31\linewidth}
      \includegraphics[width=1\linewidth]{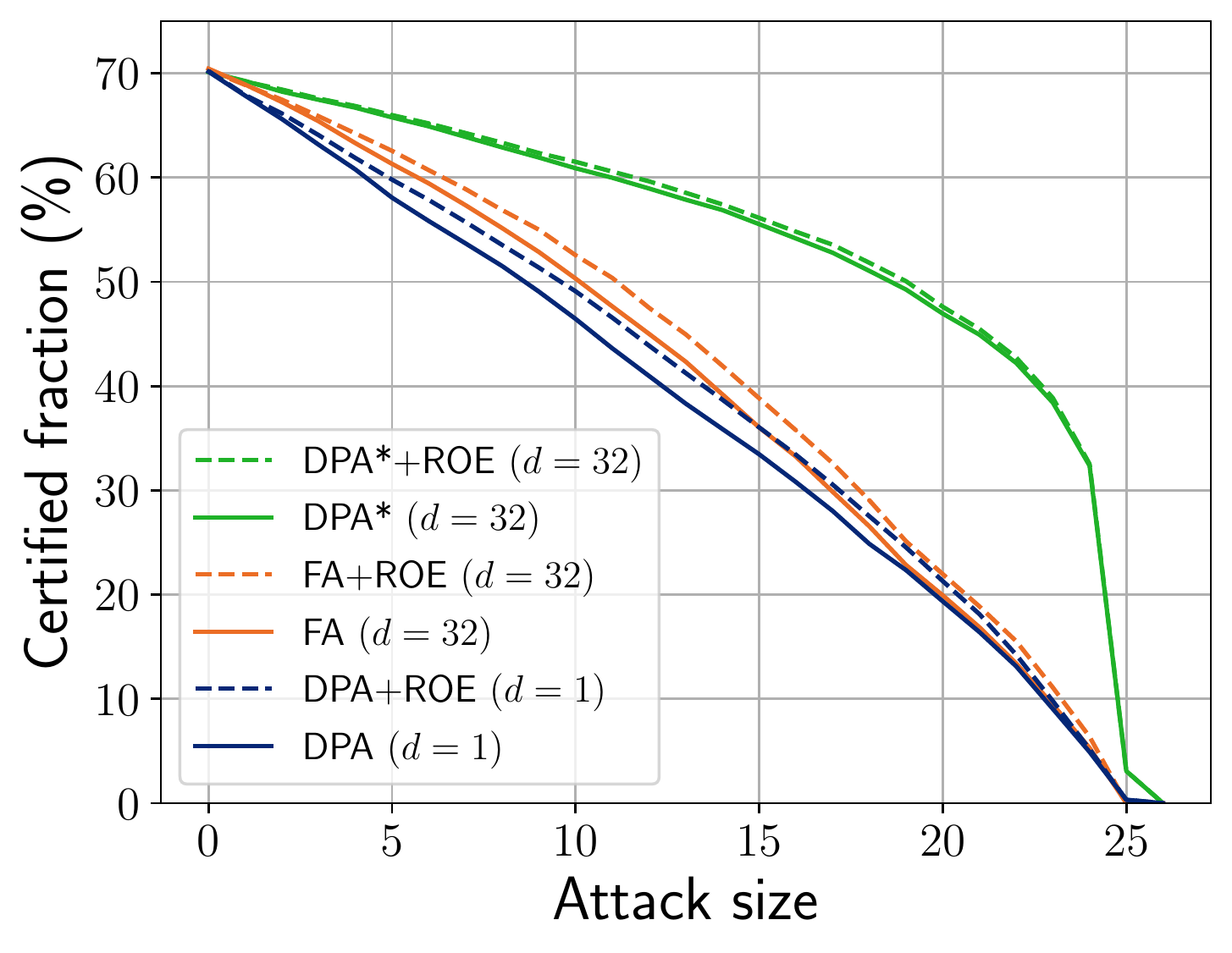}
    \end{minipage} \hspace{30pt}
    \begin{minipage}{0.31\linewidth}
      \includegraphics[width=1\linewidth]{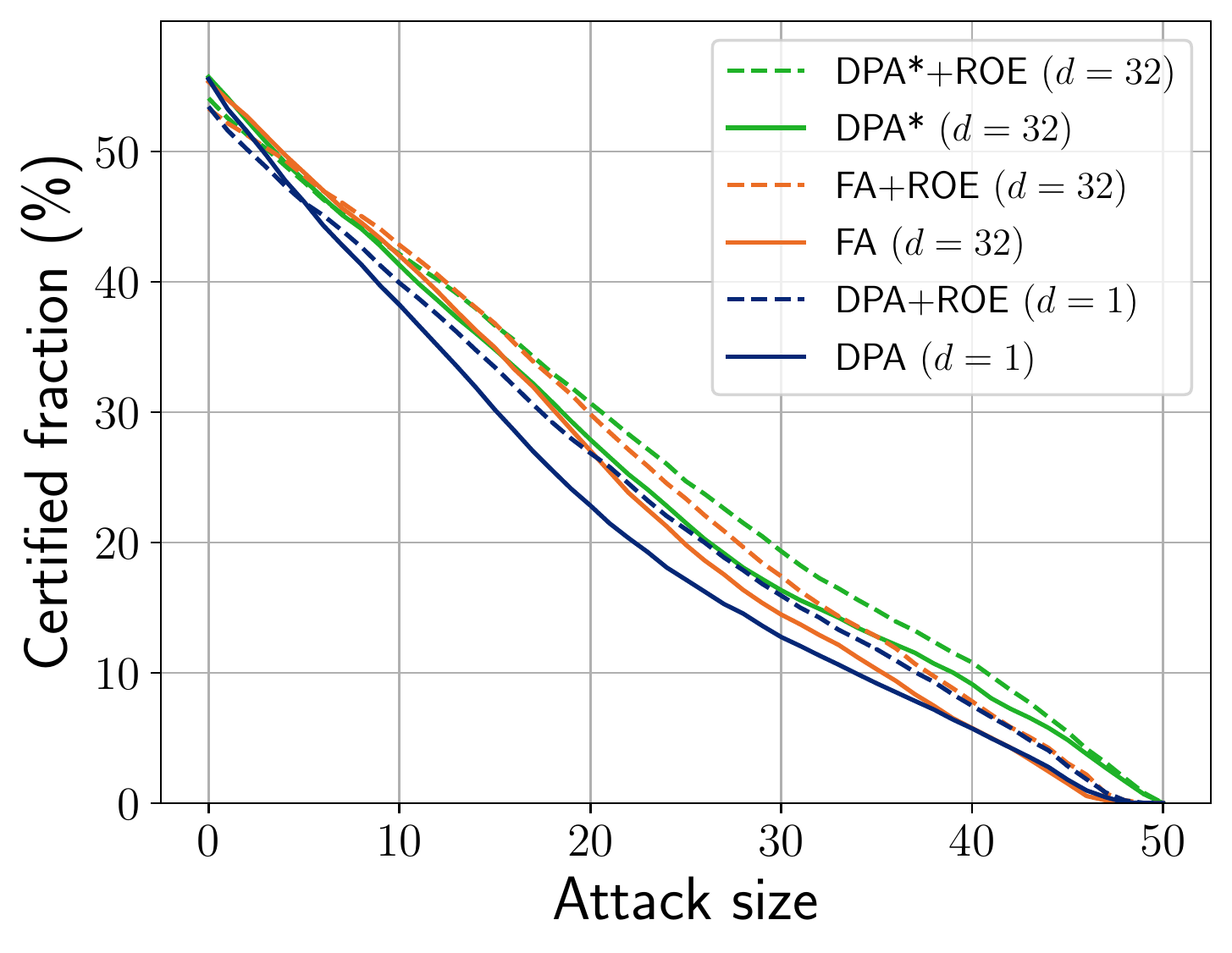}
    \end{minipage}
    \label{fig:app_certified_accuracy}
  \end{subfigure}\\
  \begin{subfigure}{1\linewidth}
    \centering
    \begin{minipage}{0.31\linewidth}
      \includegraphics[width=1\linewidth]{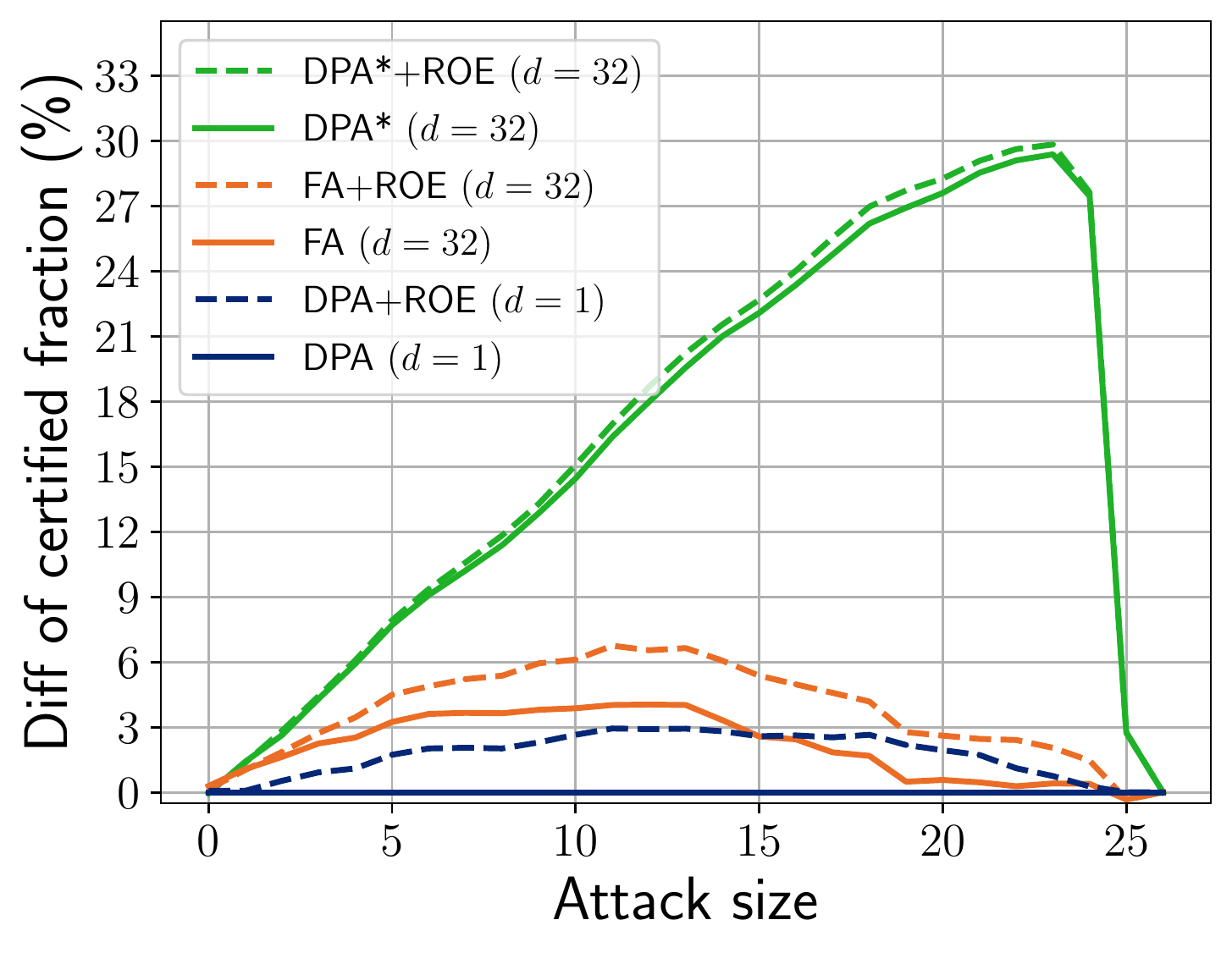}
    \end{minipage} \hspace{30pt}
    \begin{minipage}{0.31\linewidth}
      \includegraphics[width=1\linewidth]{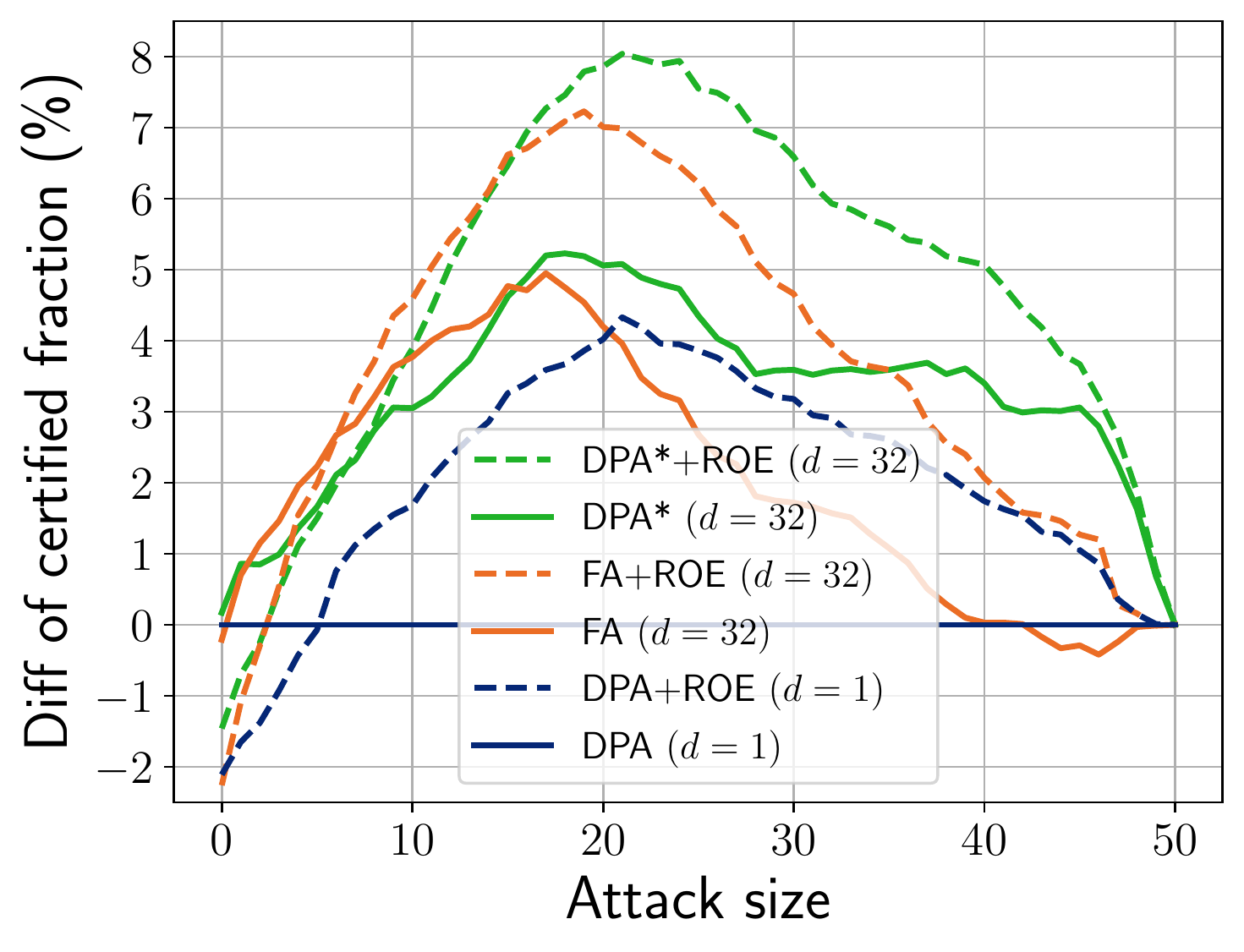}
    \end{minipage}
    \label{fig:app_improvement}
  \end{subfigure}\\
\caption{
\textbf{First row}:
The curves of certified fraction of different methods on different datasets.
\textbf{Second row}:
The improvements of certified fraction over DPA. 
Plots in the first columns refers to CIFAR-10 ($k=50$),
plots in the second column refers to GRSTB ($k=100$).
}
\label{fig:app_certified_fraction}
\end{figure*}

\section{Pseudocodes}
In this section, we provide pseudocodes that
were omitted from the main text. 
\begin{algorithm}[H]
  \caption{ROE algorithm}
  \label{alg:ROE}
  \begin{algorithmic}[1]
    \Require{Trianed classifiers $\{f_i\}_{i=1}^k$, test sample $x \in \mathcal{X}$.}
    \Ensure{Prediction of Run-off election for $x$.}
    \Function{\Predict}{$\{f_i\}_{i=1}^k, x$}
      \State \textcolor{gray}{/* Round 1 */}
      \For{$c\in \mathcal{C}$}
        \State $\none_c \gets \sum_{i=1}^k\ind{f_i(x) = c}$
      \EndFor
      \State $c_1 \gets \argmax_{c} \none_c, \quad c_2 \gets \argmax_{c\ne c_1} \none_c$.
      \State \textcolor{gray}{/* Round 2 */}
      \State $\nrun_{c_1} \gets \sum_{i=1}^k \ind{f_i^{\logit}(x, c_1) > f_i^{\logit}(x, c_2)}$
      \State $\nrun_{c_2} \gets k - \nrun_{c_2}$
      \State \Return $\argmax_{c\in \{c_1, c_2\}} \nrun_{c}$
    \EndFunction
  \end{algorithmic}
\end{algorithm}

\begin{algorithm}[H]
  \caption{Training of classifiers in \dparoe{}}
  \label{alg:DPA}
  \begin{algorithmic}[1]
    \Require{Dataset $D \subseteq \mathcal{X} \times \mathcal{C}$, hash function $h : \mathcal{X} \to [k]$, test sample $x \in \mathcal{X}$.}
    \Ensure{Prediction of DPA for $x$.}
    \Function{\Train}{$D, h$}
      \For{$i\gets 1, \dots, k$}
        \State $D_i \gets \{x \in D: h(x) = i\}, \quad f_i \gets f_{D_i}$
      \EndFor
      \State \Return $\{f_i\}_{i=1}^{k}$
    \EndFunction
  \end{algorithmic}
\end{algorithm}

\begin{algorithm}[H]
  \caption{Training of classifiers in \faroe{}}
  \label{alg:FA}
  \begin{algorithmic}[1]
    \Require{Dataset $D \subseteq \mathcal{X} \times \mathcal{C}$, hash functions $\hsplit : \mD \to [kd], \hspread: [kd] \to [kd]^d$.}
    \Ensure{Trained $\{f_i\}_{i=1}^{kd}$ classifiers.}
    \Function{\Train}{$D, \hsplit, \hspread$}
      \For{$i\gets 1, \dots, kd$}
        \State $D_{i} \gets \{x: i \in \hspread(\hsplit(x))\}, \quad f_i \gets f_{D_i}$
      \EndFor
      \State \Return $\{f_i\}_{i=1}^{kd}$
    \EndFunction
  \end{algorithmic}
\end{algorithm}

\begin{algorithm}[H]
  \caption{\fcertwo{} algorithm for \faroe{}}
  \label{alg:fa-cert}
  \begin{algorithmic}[1]
    \Require{Array $\pw$ of size $kd$, an integer $\gap$.}
    \Ensure{Minimum number of poisoned samples needed to make the gap non-positive.}
    \Function{\CertFA}{$\pw, \gap$}
    \State Sort array \pw\ in decreasing order
    \State $\mycount \gets 0$
    \While {\gap $> 0$}
      \State \gap $\gets$ \gap $- \pw_{\mycount}$
      \State $\mycount \gets \mycount + 1$
    \EndWhile
    \State \Return $\mycount$
    \EndFunction
  \end{algorithmic}
\end{algorithm}

\section{Effect of lucky seed}
\label{app:repeat}
We note that everything in DPA, FA, \dpastar{}, and our method is deterministic, so when base classifiers are fixed, then all certificates are deterministic.
Same as existing work, we evaluated our method in as many different settings as possible, i.e., we evaluated our method on different datasets, different values for the number of partitions ($k$), and different values of $d$.
In our experiments, we have noticed the error bars are very small (thus we focused more on different settings instead of repeating experiments).
To show this empirically, we run multiple trials on the CIFAR-10 dataset with $k=50$, on both DPA and FA.
Results can be seen in Table~\ref{table:repeat}. Error bars are negligible.

\begin{table}[t]
    \caption{
        Average certified accuracy of DPA, FA, \dparoe{}, and \faroe{} on CIFAR-10 with $k=50$ partitions. 
        Results of DPA and \dparoe{} are averaged over \textbf{$16$ trials} and
        results of FA and \faroe{} are reported when $d=16$ over \textbf{$4$ trials.}
        (certified accuracy is reported in the form of mean $\pm$ std)
    }
    \vskip 0.15in
    \label{table:repeat}
    \resizebox{\columnwidth}{!}{
    \begin{tabular}{|c|c|c|c|c|c|}
        \hline \hline
        \multicolumn{6}{|c|}{average of certified fraction over multiple trials} \\ 
        \hline \hline
        method & $B \leq 5$ & $B \leq 10$ & $B \leq 15$ & $B \leq 18$ & $B \leq 22$\\ 
        \hline
        DPA & $58.63\%\ (\pm 0.20\%)$ & $46.45\%\  (\pm 0.20\%)$ & $33.43\%\  (\pm 0.21\%)$ & $25.26\%\ (\pm 0.17\%)$ & $19.54\%\  (\pm 0.22\%)$\\ \hline
        \dparoe{} & $60.00\%\  (\pm 0.19\%)$ & $49.14\%\  (\pm 0.21\%)$ & $36.34\%\  (\pm 0.22\%)$ & $27.65\%\ (\pm 0.22\%)$ & $21.49\%\  (\pm 0.22\%)$\\ \hline \hline
        FA & $60.21\%\ (\pm 0.25\%)$ & $48.49\%\  (\pm 0.26\%)$ & $34.27\%\  (\pm 0.24\%)$ & $25.44\%\ (\pm 0.12\%)$ & $19.92\%\  (\pm 0.06\%)$\\ \hline
        \faroe{} & $61.48\%\ (\pm 0.18\%)$ & $50.90\%\ (\pm 0.20\%)$ & $37.16\%\  (\pm 0.03\%)$ & $28.45\%\  (\pm 0.11\%)$ & $22.07\%\ (\pm 0.09\%)$\\ \hline
    \end{tabular}}
\end{table}

\newcommand{\g}{\textnormal{g}}
\newcommand{\vonecone}{c_1}
\newcommand{\vonectwo}{c_2}
\newcommand{\vtwocone}{c}
\newcommand{\vtwoctwo}{c_1}
\newcommand{\vtwocthree}{c_2}
\newcommand{\allf}{\{f_i\}_{i=1}^{k}}
\newcommand{\mdp}{\textnormal{dp}}
\newcommand{\gsum}{\g_{\vtwoctwo + \vtwocthree}}

\section{Proofs}
In this section,
we first provide some basic lemmas.
Secondly, we prove Theorem~\ref{thm:cert}.
Then, we analyze how to calculate the certificate in \dparoe{} and \faroe{}
by proving lemmas ~\ref{lem:certv1-dpa}, \ref{lem:certv2-dpa}, \ref{lem:certv1-fa}, \ref{lem:certv2-fa} and another lemma which is provided later.

\subsection{Preliminaries}

\begin{lemma}
\label{lem:1}
For any two different classes $c_1, c_2 \in \mC$, $c_1$ beats $c_2$ if and only if $\gap(\allf, c_1, c_2) > 0$. 
\end{lemma}
\begin{proof}
Let $N_c$ denote the number of votes that class $c$ has, i.e., $N_c = \sum_{i=1}^{k} \ind{f_i(x) = c}$.
If $c_1$ beats $c_2$, then either $N_{c_1} > N_{c_2}$, or $N_{c_1} = N_{c_2}$ and $c_1 < c_2$.
Therefore $\gap(\allf, c_1, c_2) = N_{c_1} - N_{c_2} + \ind{c_2 > c_1} > 0$.

If $\gap(\allf, c_1, c_2) = N_{c_1} - N_{c_2} + \ind{c_2 > c_1} > 0$, then either $N_{c_1} > N_{c_2}$, or $N_{c_1} = N_{c_2}$ and $c_2 > c_1$.
Therefore, $c_1$ is dominant over $c_2$ and beats $c_2$.
\end{proof}

\begin{lemma}
\label{lem:2}
If the adversary wants to change the prediction of models such that
class $c_2$ beats class $c_1$,
then he needs to make sure that $\gap(\allf, c_1, c_2)$ becomes non-positive, i.e., $\gap(\allf, c_1, c_2) \leq 0$.
\end{lemma}
\begin{proof}
According to Lemma~\ref{lem:1},
after the adversary poisons models such that $c_2$ beats $c_1$, 
$\gap(\allf, c_2, c_1) > 0$.
Now, we show that it further implies that $\gap(\allf, c_1, c_2) \leq 0$.
Since $\gap(\allf, c_2, c_1) > 0$, $N_{c_2} - N_{c_1} + \ind{c_1 > c_2} > 0$. There are two cases:
\begin{itemize}
    \item \textbf{$c_1 > c_2$}. In this case, $N_{c_2} - N_{c_1} \geq 0$, which further implies that
    \begin{align*}
        \gap(\allf, c_1, c_2) = 
        N_{c_1} - N_{c_2} + \ind{c_2 > c_1} =
        N_{c_1} - N_{c_2} \le 0
    \end{align*}
    \item \textbf{$c_2 > c_1$}. In this case, $N_{c_2} - N_{c_1} > 0$, which further implies that
    \begin{align*}
        \gap(\allf, c_1, c_2) = 
        N_{c_1} - N_{c_2} + \ind{c_2 > c_1} =
        N_{c_1} - N_{c_2} + 1 \le 0
    \end{align*}
\end{itemize}
\end{proof}

\subsection{Proof of Theorem \ref{thm:cert}}
\label{pf_thm:cert}
\begin{proof}
We consider how the adversary can change the prediction of our aggregation method on sample $x$. Note that we further eliminate $x$ from notations for the sake of simplicity. 
Let $\cpred$ denote the predicted class and $\csec$ denote the other selected class in Round 1.
The adversary should ensure that either $\cpred$ is not selected as the top two classes in Round 1 or if it makes it to Round 2, it loses this round to another class. 

We start with the first strategy. As $\cpred$ should be eliminated from the top-two classes of Round 1, it means that the adversary needs to 
choose two different classes $c_1, c_2\in \mC \backslash \{\cpred\}$ and ensure
that $c_1$ and $c_2$ both are dominant over $\cpred$ in this round.
This is the exact definition of $\fcerthree$, i.e., it needs at least 
$\fcerthree(\cpred, c_1, c_2)$ poisoned samples. As the adversary can choose classes $c_1, c_2$, eliminating $\cpred$ in Round 1 requires it at least $\cert^{\rone}$ poisoned sample where
\begin{align*}
    \cert^{\rone} := \min_{c_1, c_2\in \mC \backslash \{\cpred\}}\fcerthree(\cpred, c_1, c_2).
\end{align*}

Next, in the second strategy, the adversary ensures that class $c'$ beats $\cpred$ in Round 2. To do so, it needs to poison models such that (a) $c'$ is selected in the top-two classes of Round 1, (b) $c'$ beats $\cpred$ in Round 2.

For (a), $c'$ should beat either of $\cpred$ or $\csec$ in Round 1. 
As we have already considered the case that $\cpred$ is eliminated in Round 1, we focus on the case that $c'$ needs to beat $\csec$. According to the definition of 
$\fcertwo$, this requires at least $\cert^{\rtwo}_{c', 1}$ samples where 
\begin{align*}
    \cert^{\rtwo}_{c', 1} := \fcertwo(\csec, c')
    .
\end{align*}
Note that $c'$ can be $\csec$.

For (b), let $g_i^{c'} : \mathcal{X} \to \{c', \cpred\}$ denote the binary $\cpred$ vs $c'$ classifier obtained from $f_i$, i.e.,
\begin{align*}
    g_i^{c'}(x) := \begin{cases}
        c' \quad&\text{if } f_i^{\logit}(x, c') > f_i^{\logit}(x, \cpred)\\
        \cpred \quad&\text{otherwise}
    \end{cases}.
\end{align*}
The adversary needs to ensure that in this binary classification problem, class $c'$ beats $\cpred$. This is the definition of $\fcertwo$, i.e., it requires at least $\cert^{\rtwo}_{c', 2}$ poisoned samples where 
\begin{align*}
    \cert^{\rtwo}_{c', 2} := 
    \fcertwo{}(\{g_i^{c'}\}_{i=1}^{k}, \cpred, c').
\end{align*}
Overall, since the adversary can choose the class $c'$, we obtain the bound
\begin{align*}
    \cert^{\rtwo} := \min_{c' \ne \cpred} \max\{\cert^{\rtwo}_{c', 1}, \cert^{\rtwo}_{c', 2}\}.
\end{align*}
\end{proof}    

Now we explain how $\fcertwo$ and $\fcerthree$ can be efficiently calculated
in both \dparoe{} and \faroe{}.

\subsection{Deep Partition Aggregation + Run-off Election}
\label{sec:app_roe_dpa}
We first prove Lemma~\ref{lem:certv1-dpa} which calculates the value of
$\fcertwo$ in \dparoe.
After that, we focus on how to calculate $\fcerthree$ in this method by proving Lemma~\ref{lem:certv2-dpa}.

\subsubsection{Proof of Lemma \ref{lem:certv1-dpa}}
\label{pf_lem:certv1-dpa}
\begin{proof}
We want to find the value of $\fcertwo(\allf, \vonecone, \vonectwo)$.
Based on Lemma~\ref{lem:2}, in order to ensure that $\vonectwo$ beats $\vonecone$, $\gap(\vonecone, \vonectwo)$ should become non-positive, i.e., $\gap(\vonecone, \vonectwo) \le 0$. Now we consider how poisoning each partition can change the $\gap$. For simplicity, we show this $\gap$ with $\g$.

When poisoning partition $D_i$, the adversary can change the prediction of $f_i$ to be whatever it wants. 
We will consider how the adversary can change $\g$ by fooling  model $f_i$.
Note that by poisoning $D_i$, none of the other classifiers change.

\begin{enumerate}
    \item \textbf{$f_i(x) = \vonecone$}.
    In this case, if the adversary fools this model and changes its prediction to $\tilde{c} \neq \vonecone$, we have two cases:
    \begin{itemize}
        \item $\tilde{c} = \vonectwo$. In this case, $\g$ decreases by $2$. 
        \item $\tilde{c} = c'$. In this case, $\g$ decreases by $1$.
    \end{itemize}
    \item \textbf{$f_i(x) = \vonectwo$}.
    In this case, if the adversary changes the prediction to $\tilde{c} \neq \vonectwo$, we have two cases:
    \begin{itemize}
        \item $\tilde{c} = \vonecone$. In this case, $\g$ increases by $2$.
        \item $\tilde{c} = c'$. In this case, $\g$ increases by $1$.
    \end{itemize}
    \item \textbf{$f_i(x) = c'$} where $c' \notin \{\vonecone, \vonectwo\}$
    In this case, if the adversary changes the prediction to $\tilde{c} \neq c'$, we have three cases:
    \begin{itemize}
        \item $\tilde{c} = \vonecone$. In this case, $\g$ increases by $1$.
        \item $\tilde{c} = \vonectwo$. In this case, $\g$ decreases by $1$.
        \item $\tilde{c} = c''$. In this case, $\g$ does not change.
    \end{itemize}
\end{enumerate}

As seen above, by poisoning a~single partition, the adversary can reduce $\g$ by at most $2$.
Hence, ensuring $\g \le 0$ requires at least $\max\left(0, \ceil{\frac{\g}{2}}\right)$ poisoned samples. This finishes the proof.
\end{proof}

\subsubsection{Proof of Lemma \ref{lem:certv2-dpa}}
\label{pf_lem:certv2-dpa}
\begin{proof}
    We want to find the value of $\fcerthree(\allf, \vtwocone, \vtwoctwo, \vtwocthree)$.
    Based on Lemma~\ref{lem:2}, in order to ensure that both $\vtwoctwo$ and $\vtwocthree$ beat $\vtwocone$,
    both $\gap(\vtwocone, \vtwoctwo)$ and $\gap(\vtwocone, \vtwocthree)$ should become non-positive. For simplicity, we denote those gaps by $\g_1$ and $\g_2$, respectively.

    When poisoning partition $D_i$, the adversary can change the prediction of $f_i$ to be whatever it wants. 
    Now we consider the effect of poisoning model $f_i$ on $\g_1$ and $\g_2$.
    Note that by poisoning $D_i$, none of the other classifiers change.

    \begin{enumerate}
        \item \textbf{$f_i(x) = \vtwocone$}.
        In this case, if the adversary fools this model and changes its prediction to $\tilde{c} \neq \vtwocone$, we have three cases:
        \begin{itemize}
            \item $\tilde{c} = \vtwoctwo$. In this case, $\g_1$ decreases by $2$ while $\g_2$ decreases by $1$ \textbf{(\textit{type (i)})}.
            \item $\tilde{c} = \vtwocthree$. In this case, $\g_1$ decreases by $1$ while $\g_2$ decreases by $2$ \textbf{(\textit{type (ii)})}.
            \item $\tilde{c} = c'$ where $c' \notin \{\vtwoctwo, \vtwocthree\}$.
            In this case, both $\g_1$ and $\g_2$ are reduced by $1$.
        \end{itemize}
        \item \textbf{$f_i(x) = \vtwoctwo$}.
        In this case, if the adversary changes the prediction to $\tilde{c} \neq \vtwoctwo$, we have three cases:
        \begin{itemize}
            \item $\tilde{c} = \vtwocone$. In this case, $\g_1$ increases by $2$ while $\g_2$ increases by $1$.
            \item $\tilde{c} = \vtwocthree$. In this case, $\g_1$ increases by $1$ while $\g_2$ decreases by $1$.
            \item $\tilde{c} = c'$ where $c' \notin \{\vtwocone, \vtwocthree\}$. In this case, $\g_1$ increases by $1$ while $\g_2$ remains same.
        \end{itemize}
        \item \textbf{$f_i(x) = \vtwocthree$}.
        In this case, if the adversary changes the prediction to $\tilde{c} \neq \vtwocthree$, we have three cases:
        \begin{itemize}
            \item $\tilde{c} = \vtwocone$. In this case, $\g_1$ increases by $1$ while $\g_2$ increases by $2$.
            \item $\tilde{c} = \vtwoctwo$. In this case, $\g_1$ decreases by $1$ while $\g_2$ increases by $1$.
            \item $\tilde{c} = c'$ where $c' \notin \{\vtwocone, \vtwoctwo\}$. In this case, $\g_1$ remains same while $\g_2$ increases by $1$.
        \end{itemize}
        \item \textbf{$f_i(x) = c'$} where $c' \notin \{\vtwocone, \vtwoctwo, \vtwocthree\}$
        In this case, if the adversary changes the prediction to $\tilde{c} \neq c'$, we have four cases:
        \begin{itemize}
            \item $\tilde{c} = \vtwocone$. In this case, both $\g_1$ and $\g_2$ increase by $1$.
            \item $\tilde{c} = \vtwoctwo$. In this case, $\g_1$ decreases by $1$ while $\g_2$ remains same.
            \item $\tilde{c} = \vtwocthree$. In this case, $\g_1$ remains the same while $\g_2$ decreases by $1$.
            \item $\tilde{c} = c''$ where $c'' \notin \{\vtwocone, \vtwoctwo, \vtwocthree\}$. In this case, none of $\g_1$ and $\g_2$ change.
        \end{itemize}
    \end{enumerate}
    
As the adversary's goal is to make both $\g_1$ and $\g_2$ non-positive with the minimum number of poisoned samples,
based on the scenarios above,
\textbf{by poisoning a single model}, 
its power is bounded either with \textit{type (i)} decreasing $\g_1$ by $2$ and decreasing $\g_2$ by $1$,
or with \textit{type (ii)} decreasing $\g_1$ by $1$ and decreasing $\g_2$ by $2$.

Now we use induction to prove that array $\mdp$ given in the lemma statement,
can find a~lower~bound on the minimum number of poisoned samples.
For base cases, we consider two cases when $\min(\g_1, \g_2) \le 1$.
\begin{itemize}
    \item $\g_1 = \g_2 = 0$. In this case, no poisoned samples needed so $\mdp[0, 0] = 0$
    \item $\max(\g_1, \g_2) > 0$ and $\min(\g_1, \g_2) \le 1$. In this case, the adversary needs at least one poisoned sample.
    Furthermore, by one poisoned sample, both $\g_1$ and $\g_2$ can be reduced by at most $2$.
    Hence, we need at least $\ceil{\frac{\max(\g_1, \g_2)}{2}}$ poisoned samples. This implies that $\mdp[\g_1, \g_2]=\ceil{\frac{\max(\g_1, \g_2)}{2}}$.
\end{itemize}

Now we find a~lower~bound when $i = \g_1$ and $j = \g_2$. Note that $i, j \ge 2$. The adversary has two options when using one poisoned sample.
(1) It reduces $i$ by $2$ and $j$ by $1$, then according to induction, it needs at least $\mdp[i-2, j-1]$ more poisoned samples.
(2) It reduces $i$ by $1$ and $j$ by $2$. Hence it requires at least $\mdp[i-1, j-2]$ more poisoned samples.

As a result, the minimum number of poisoned samples is at least $\mdp[i, j] = 1 + \min(\mdp[i-2, j-1], \mdp[i-1, j-2])$.

This finishes the proof.
\end{proof}

\subsection{Finite Aggregation + Run-off Election}
In this part, we first provide a lemma that we use to prove Lemmas~\ref{lem:certv1-fa} and \ref{lem:certv2-fa}.

\begin{lemma}
\label{lem:3}
    Given the array $(\pw_b)_{b \in [kd]}$ which represents the adversary's power in terms of reducing the gap $\g$.
    Let $\pi = (\pi_1, \pi_2, ..., \pi_{kd})$ be a~permutation on $[kd]$ such that
    $\pw_{\pi_1} \ge \pw_{\pi_2} \ge ... \ge \pw_{\pi_{kd}}$,
    the adversary needs to poison at least $t$ buckets to make $\g \le 0$ where $t$ is the minimum
    non-negative integer such that $\sum_{i=1}^{t} \pw_{\pi_i} \ge \g$.
    Furthermore, $t = \textsc{CertFA}((\pw_b)_{b \in [kd]}, \g)$.
\end{lemma}
\begin{proof}
    Let $B = \{b_1, b_2, ..., b_{t'}\}$ be a set of buckets that if the adversary poisons them,
    can ensure that $\g \le 0$. We show that $t \le t'$.
    Since poisoning bucket $b_i$ can reduce the gap by $\pw_{b_i}$, 
    poisoning buckets in $B$ can reduce the gap by \textbf{at most} $\sum_{i=1}^{t'} \pw_{b_i}$.
    The reason we say "at most" is the fact that
    a base classifier uses several buckets in its training sets,
    so by poisoning more than one of those buckets, 
    we count the effect of poisoning that particular classifier several times. 
    As $\pi$ sorts the array in decreasing order, $\sum_{i=1}^{t'} \pw_{\pi_i} \ge \sum_{i=1}^{t'} \pw_{b_i}$.
    This implies that $\sum_{i=1}^{t'} \pw_{\pi_i} \ge \g$.
    According to the definition of $t$ which is
    the minimum non-negative integer such that $\sum_{i=1}^{t} \pw_{\pi_i} \ge \g$, we conclude that $t \le t'$.
    In order to find $t$, we sort array $\pw$ in decreasing order
    and while the sum of the elements we have picked does not reach $\g$,
    we keep picking elements, moving from the left side to the right side of the array.
    A pseudocode of how to find $t$ is given in Algorithm~\ref{alg:fa-cert}, which further proves that 
    $t = \textsc{CertFA}((\pw_b)_{b \in [kd]}, \g)$.
    
    Note that $t \le kd$ always exists as fooling all models to do in favor of a desired class guarantees that the class will beat all other classes.
\end{proof}

\subsubsection{Proof of Lemma \ref{lem:certv1-fa}}
\label{pf_lem:certv1-fa}
\begin{proof}
    We can have the same argument of Section~\ref{pf_lem:certv1-dpa} so 
    in order to ensure that $\vonectwo$ beats $\vonecone$, $\gap(\vonecone, \vonectwo)$ should become non-positive, i.e., $\gap(\vonecone, \vonectwo) \le 0$. Now we consider how poisoning each bucket can change the $\gap$. For simplicity, we show this $\gap$ with $\g$.

    Let $A$ be the set of indices of classifiers that can be fooled after poisoning $b$.
    Formally, $A := \hspread(b)$.
    We consider $j \in A$ and examine how $\g$ changes as $f_j$ is poisoned.
    \begin{enumerate}
        \item \textbf{$f_j(x) = \vonecone$}.
        In this case, if the adversary fools this model and changes its prediction to $\tilde{c} \neq \vonecone$, we have two cases:
        \begin{itemize}
            \item $\tilde{c} = \vonectwo$. In this case, $\g$ decreases by $2$.
            \item $\tilde{c} = c'$ where $c' \neq \vonectwo$. In this case, $\g$ is reduced by $1$.
        \end{itemize}
        This implies that in this case, in terms of reducing $\g$, the adversary's power is \textbf{bounded by $2$}.
        \item \textbf{$f_j(x) = \vonectwo$}.
        In this case, if the adversary changes the prediction to $\tilde{c} \neq \vonectwo$, we have two cases:
        \begin{itemize}
            \item $\tilde{c} = \vonecone$. In this case, $\g$ increases by $2$.
            \item $\tilde{c} = c'$ where $c' \neq \vonecone$. In this case, $\g$ increases by $1$.
        \end{itemize}
        This implies that in this case, in terms of reducing $\g$, the adversary's power is \textbf{bounded by $0$}.
        \item \textbf{$f_i(x) = c'$}.
        In this case, if the adversary changes the prediction to $\tilde{c} \neq c'$, we have three cases:
        \begin{itemize}
            \item $\tilde{c} = \vonecone$. In this case, $\g$ increases by $1$.
            \item $\tilde{c} = \vonectwo$. In this case, $\g$ decreases by $1$.
            \item $\tilde{c} = c''$ where $c'' \notin \{\vonecone, \vonectwo\}$. In this case, $\g$ remains the same.
        \end{itemize}
        This implies that in this case, in terms of reducing $\g$, the adversary's power is \textbf{bounded by $1$}.
    \end{enumerate}

    Note that we can have the same argument for all $j \in A$.
    According to the above cases, We define the poisoning power of bucket $b$ with respect to classes $\vonecone, \vonectwo$ as 
    \begin{align*}
      \pw_{\vonecone, \vonectwo, b} :=
      \sum_{j \in \hspread(b)} 2\ind{f_j(x) = \vonecone} + \ind{f_j(x) \notin \{\vonecone, \vonectwo\}}
      .
    \end{align*}
    Using Lemma~\ref{lem:3}, it is obvious that $\fcertwo(\allf, \vonecone, \vonectwo) := \textsc{CertFA}((\pw_{\vonecone, \vonectwo, b})_{b \in [kd]}, \g)$
    is a 1v1 certificate.
\end{proof}

\subsubsection{Proof of Lemma \ref{lem:certv2-fa}}
\label{pf_lem:certv2-fa}
\begin{proof}
We can exactly follow the same initial steps of proof in Lemma~\ref{lem:certv2-dpa}.
Therefore, in order to ensure that both $\vtwoctwo$ and $\vtwocthree$ beat $\vtwocone$,
both $\gap(\vtwocone, \vtwoctwo)$ and $\gap(\vtwocone, \vtwocthree)$ should become non-positive. For simplicity, we denote those gaps by $\g_1$ and $\g_2$, respectively.
Furthermore, as both $g_1$ and $\g_2$ should become non-positive, their sum which we denote by $\gsum := \g_1 + \g_2$ should become non-positive as well.

Now, we analyze how the adversary can affect $\g_1$, $\g_2$, and $\gsum$ by poisoning a bucket $b$.
Let $A$ be the set of indices of classifiers that can be fooled after poisoning $b$.
Formally, $A := \hspread(b)$.
We consider $j \in A$ and examine how $\g_1, \g_2$, and $\gsum$ change as $f_j$ is poisoned.


\begin{enumerate}
    \item \textbf{$f_j(x) = \vtwocone$}.
    In this case, if the adversary fools this model and changes its prediction to $\tilde{c} \neq \vtwocone$, we have three cases:
    \begin{itemize}
        \item $\tilde{c} = \vtwoctwo$. In this case, $\g_1$ decreases by $2$, $\g_2$ decreases by $1$, and $\gsum$ is reduced by $3$.
        \item $\tilde{c} = \vtwocthree$. In this case, $\g_1$ decreases by $1$, $\g_2$ decreases by $2$, and $\gsum$ is reduced by $3$.
        \item $\tilde{c} = c'$ where $c' \notin \{\vtwoctwo, \vtwocthree\}$. In this case, both $\g_1$ and $\g_2$ are reduced by $1$ while $\gsum$ is reduced by $2$.
    \end{itemize}
    This implies that in terms of reducing $\g_1$, $\g_2$, and $\gsum$, adversary's power is \textbf{bounded by $2, 2$, and $3$, respectively}.
    \item \textbf{$f_j(x) = \vtwoctwo$}.
    In this case, if the adversary changes the prediction to $\tilde{c} \neq c'_1$, we have three cases:
    \begin{itemize}
        \item $\tilde{c} = \vtwocone$. In this case, $\g_1$ increases by $2$, $\g_2$ increases by $1$, and $\gsum$ increases by $3$.
        \item $\tilde{c} = \vtwocthree$. In this case, $\g_1$ increases by $1$, $\g_2$ decreases by $1$, and $\gsum$ does not change.
        \item $\tilde{c} = c''$ where $c'' \notin \{\vtwocone, \vtwocthree\}$. In this case, $\g_1$ increases by $1$, $\g_2$ remains same, and $\gsum$ increases by $1$.
    \end{itemize}
    This implies that in terms of reducing $\g_1$, $\g_2$, and $\gsum$, adversary's power is \textbf{bounded by $0, 1$, and $0$, respectively}.
    \item \textbf{$f_j(x) = \vtwocthree$}.
    In this case, if the adversary changes the prediction to $\tilde{c} \neq \vtwocthree$, we have three cases:
    \begin{itemize}
        \item $\tilde{c} = \vtwocone$. In this case, $\g_1$ increases by $1$, $\g_2$ increases by $2$, and $\gsum$ increases by $3$.
        \item $\tilde{c} = \vtwoctwo$. In this case, $\g_1$ decreases by $1$, $\g_2$ increases by $1$, and $\gsum$ does not change.
        \item $\tilde{c} = c'$ where $c' \notin \{\vtwocone, \vtwoctwo\}$. In this case, $\g_1$ remains same, $\g_2$ increases by $1$, and $\gsum$ increases by $1$.
    \end{itemize}
    This implies that in terms of reducing $\g_1$, $\g_2$, and $\gsum$, adversary's power is \textbf{bounded by $1, 0$, and $0$, respectively}.
    \item \textbf{$f_j(x) = c'$} where $c' \notin \{\vtwocone, \vtwoctwo, c'_2\}$
    In this case, if the adversary changes the prediction to $\tilde{c} \neq c''$, we have four cases:
    \begin{itemize}
        \item $\tilde{c} = \vtwocone$. In this case, both $\g_1$ and $\g_2$ increase by $1$ while $\gsum$ increases by $2$.
        \item $\tilde{c} = \vtwoctwo$. In this case, $\g_1$ decreases by $1$, $\g_2$ remains same, and $\gsum$ decreases by $1$.
        \item $\tilde{c} = \vtwocthree$. In this case, $\g_1$ remains the same, $\g_2$ decreases by $1$, and $\gsum$ decreases by $1$.
        \item $\tilde{c} = c''$ where $c'' \notin \{\vtwocone, \vtwoctwo, \vtwocthree\}$. In this case, none of $\g_1$, $\g_2$, and $\gsum$ change.
    \end{itemize}
    This implies that in terms of reducing $\g_1$, $\g_2$, and $\gsum$, adversary's power is \textbf{bounded by $1, 1$, and $1$, respectively}.
\end{enumerate}

Note that we can have the same argument for all $j \in A$.

In terms of reducing $\g_1$,  we define the \emph{poisoning power of bucket $b$}, i.e.,
the maximum amount that $\g_1$ can be reduce by poisoning $b$ as $\pw_{\vtwocone, \vtwoctwo, b}$.
Based on the scenarios above, 

\begin{align*}
  \pw_{\vtwocone, \vtwoctwo, b} :=
  \sum_{j \in \hspread(b)} 2\ind{f_j(x) = \vtwocone} + \ind{f_j(x) \notin \{\vtwoctwo, c\}}
  .
\end{align*}
According to Lemma~\ref{lem:3}, the adversary needs at least $\fcerthree^{(1)}$ poisoned samples where 
\begin{align*}
    \fcerthree^{(1)} := \textsc{CertFA}((\pw_{\vtwocone, \vtwoctwo, b})_{b \in [kd]}, \g_1)
\end{align*}

Based on the scenarios above, in terms of reducing $\g_2$ to make it non-positive, 
poisoning power of bucket $b$ is at most
\begin{align*}
  \pw_{\vtwocone, \vtwocthree, b} :=
  \sum_{j \in \hspread(b)} 2\ind{f_j(x) = \vtwocone} + \ind{f_j(x) \notin \{\vtwocthree, c\}}
  .
\end{align*}
According to Lemma~\ref{lem:3}, the adversary needs at least $\fcerthree^{(2)}$ poisoned samples where 
\begin{align*}
    \fcerthree^{(2)} := \textsc{CertFA}((\pw_{\vtwocone, \vtwocthree, b})_{b \in [kd]}, \g_2)
\end{align*}

Finally, 
In terms of reducing $\gsum$, 
bucket $b$'s poisoning power is at most
\begin{align*}
  \pw^{+}_{\vtwoctwo, \vtwocthree, b} :=
  \sum_{j \in \hspread(b)} 3\ind{f_j(x) = \vtwocone} + \ind{f_j(x) \notin 
  \{c, \vtwoctwo, \vtwocthree\}}
  .
\end{align*}

This implies that the adversary is required to provide at least $\fcerthree^{+}$ poisoned samples where
\begin{align*}
    \fcerthree^{+} := \textsc{CertFA}((\pw^{+}_{\vtwoctwo, \vtwocthree, b})_{b \in [kd]}, \gsum)
    .
\end{align*}

As the adversary has to ensure that all $\g_1, \g_2$, and $\gsum$ become non-positive,
it needs to satisfy all those three conditions.
Hence $\fcerthree$ defines as follows
\begin{align*}
    \fcerthree := \max\{\fcerthree^{(1)}, \fcerthree^{(2)}, \fcerthree^{+}\}
\end{align*}
is a 2v1 certificate.
\end{proof}

\end{document}